\documentclass{article}

\usepackage{PRIMEarxiv}
\usepackage{amsthm}

\usepackage{natbib}

\usepackage{graphicx,amsmath,amsfonts,amssymb,bm,hyperref,url,epsfig,epsf,color,fullpage,MnSymbol,mathbbol, fmtcount,algorithmic,algorithm, semtrans
} 
\numberwithin{equation}{section}

\usepackage{titlesec}
\usepackage[]{mdframed}

\usepackage{tikz}
\usepackage{pgfplots}

\usetikzlibrary{pgfplots.groupplots}
\usepackage{amsmath,amssymb}

\titleformat{\paragraph}
{\normalfont\normalsize\bfseries}{\theparagraph}{1em}{}
\titlespacing*{\paragraph}
{0pt}{3.25ex plus 1ex minus .2ex}{1.5ex plus .2ex}

\usepackage{caption,subcaption}

\usepackage[bottom,hang,flushmargin]{footmisc} 
\usepackage[font=small]{caption}
\DeclareMathOperator*{\argmax}{arg\,max}

\usepackage{hyperref}
\definecolor{darkred}{RGB}{150,0,0}
\definecolor{darkgreen}{RGB}{0,150,0}
\definecolor{darkblue}{RGB}{0,0,200}
\hypersetup{colorlinks=true, linkcolor=darkred, citecolor=darkgreen, urlcolor=black}

\newtheorem{theorem}{Theorem}[section]
\newtheorem{lemma}[theorem]{Lemma}

\newtheorem{proposition}[theorem]{Proposition}
\newtheorem{definition}[theorem]{Definition}

\newtheorem{remark}{Remark}

\newtheorem{example}{Example}

\DeclareMathOperator*{\argmin}{arg\,min}

\usepackage[utf8]{inputenc} % allow utf-8 input
\usepackage[T1]{fontenc}    % use 8-bit T1 fonts
\usepackage{hyperref}       % hyperlinks
\usepackage{url}            % simple URL typesetting
\usepackage{booktabs}       % professional-quality tables
\usepackage{amsfonts}       % blackboard math symbols
\usepackage{nicefrac}       % compact symbols for 1/2, etc.
\usepackage{microtype}      % microtypography
\usepackage{lipsum}
\usepackage{fancyhdr}       % header
\usepackage{graphicx}       % graphics
\graphicspath{{media/}}     % organize your images and other figures under media/ folder

\title{Tight Rates in Supervised Outlier Transfer Learning}

\author{
  Mohammadreza M. Kalan\\
  Statistics, Columbia University \\
  \texttt{mm6244@columbia.edu} \\
   \And
  Samory Kpotufe \\
  Statistics, Columbia University\\
  \texttt{samory@columbia.edu} \\
}

\begin{document}
\maketitle

\begin{abstract}
A critical barrier to learning an accurate decision rule for outlier detection is the scarcity of outlier data. As such, practitioners often turn to the use of similar but imperfect outlier data from which they might \emph{transfer} information to the target outlier detection task. Despite the recent empirical success of transfer learning approaches in outlier detection, a fundamental understanding of when and how knowledge can be transferred from a source to a target outlier detection task remains elusive. In this work, we adopt the traditional framework of Neyman-Pearson classification---which formalizes \emph{supervised outlier detection}---with the added assumption that one has access to some related but imperfect outlier data. Our main results are as follows: 
\begin{itemize} 
\item We first determine the information-theoretic limits of the problem under a measure of discrepancy that extends some existing notions from traditional balanced classification; interestingly, unlike in balanced classification, seemingly very dissimilar sources can provide much information about a target, thus resulting in fast transfer.

\item We then show that, in principle, these information-theoretic limits are achievable by \emph{adaptive} procedures, i.e., procedures with no a priori information on the discrepancy between source and target outlier distributions. 
\end{itemize}
\end{abstract}

\section{Introduction}

A primary objective in many data science applications is to learn a decision rule that separates a common class with abundant data from a \emph{rare} class with limited or no data. This is a traditional problem which often appears under the umbrella term of \emph{outlier detection} or \emph{rare class classification}, and has seen a resurgence of interest in modern applications such as 
malware detection in cybersecurity and IoT \citep{jose2018survey, kumar2019edima}, fraud detection in credit card transactions \citep{malini2017analysis}, disease diagnosis \citep{bourzac2014diagnosis,zheng2011machine}, among others. A main goal in these applications---which distinguishes it from traditional classification where performance is asssessed on \emph{average} over all classes---is to achieve low classification error on the rare class, while at the same time maintaining low error w.r.t. the common class. Such a constrained objective is commonly referred to as Neyman-Pearson classification. Formally, letting $\mu_0, \mu_1$ denote the common and rare class distributions, Neyman-Pearson classification takes the form: 
\begin{align*}
    &\text{Minimize}\ \mu_1\text{-error}\ \text{over classifiers } {h}\  \text{in some hypothesis space }\ \cal H \\
    & \text{subject to keeping the }\  \mu_0\text{-error of such an }\ {h}\  \text{ under a threshold}\ \alpha. 
\end{align*}

In this work, we focus on the common supervised setting where practitioners have access to not only training data from the common class, but also some (limited amount of) data from the rare class or, pertinently, \emph{from a related distribution they hope has information on the rare class}. Henceforth, for simplicity of notation, we denote the \emph{target} rare class distribution by $\mu_{1, T}$ and the related but imperfect rare class distribution by $\mu_{1, S}$, where "S" stands for \emph{source}. 
As an example, such related rare-class data may be from a different infected device in IoT applications, or from similar cancer types in medical applications, or from laboratory simulations of a rare class. This is thus a \emph{transfer learning} problem, however for \emph{supervised outlier detection} rather than for traditional classification as is usual in the literature. 

While this \emph{outlier transfer} problem is quite common in applications due to the scarcity of rare class data \citep{su2023transfer,wen2019time,aburakhia2020transfer}, the problem has so far received little rigorous attention. Our main aim is therefore to further theoretical understanding of this timely problem, and in particular to gain much needed insight on the extent to which related but imperfect rare class data may improve performance for the \emph{target} Neyman-Pearson classification problem. Such achievable transfer performance of course must depend on \emph{how far} the related rare class distribution $\mu_{1, S}$ is from the target $\mu_{1, T}$---somehow properly formalized---and or whether the related rare-class distribution $\mu_{1, S}$ induces similar optimal decision rules as for the target rare class distribution $\mu_{1, T}$. 

We first argue that, unlike in traditional classification, seemingly very different source and target distributions $\mu_{1, S}, \mu_{1, T}$ may induce the same exact (universally) optimal decision rules in Neyman-Pearson classification; this is obtained as a consequence of a simple extension of the classical Neyman-Pearson Lemma \citep{lehmann1986testing} to the case of transfer (see Proposition \ref{lem6}). This is illustrated in Fig \ref{fig1} and 
explained in detail in Section \ref{full_trans}. As a consequence, unlike in usual classification, we can approximate the universally best classifier $h^*_{T,\alpha}$ under the target  arbitrarily well asymptotically, i.e., with sufficiently large data from a seemingly unrelated source.

However, the story turns out more nuanced in finite-sample regimes, i.e, as we consider the rate of such approximation (in terms of relevant sample sizes), even when $\mu_{1, S}$ and $\mu_{1, T}$ admits the same optimal classifiers. That is, two different sources $\mu_{1, S}$ and $\mu_{1, S}'$ may yield considerably different transfer rates in finite-sample regimes even if both of them share the same optimal classifiers as the target $\mu_{1, T}$: this is because a given source may yield more data near the common decision boundary $h^*_{T,\alpha}$ than another source, thus revealing $h^*_{T,\alpha}$ at a faster rate. In particular, we show in our first main result of Theorem \ref{minimax_rate}---a minimax lower bound---that the rate of convergence of \emph{any outlier-transfer approach} is in fact controlled by a relatively simple notion of \emph{outlier transfer exponent} (adapted from transfer results in traditional classification) which essentially measures how well a source may reveal the unknown decision boundary. Theorem \ref{minimax_rate} is in fact rather general: the minimax lower bound holds for \emph{any hypothesis space \emph{$\cal H$}} of finite VC dimension (at least 3), even in cases where no samples from the rare target class $\mu_{1,T}$ are available. Moreover, the result holds generally $h^*_{S,\alpha}$ and $h^*_{T,\alpha}$ are the same or not. 

We finally turn our attention to whether such rates may be achieved \emph{adaptively}, i.e., from samples alone without prior knowledge of the discrepancy between $\mu_{1, S}$ and $\mu_{1, T}$ as captured by both the transfer exponent and the \emph{amount} of difference between optimal classifiers $h^*_{S,\alpha}$ and $h^*_{T,\alpha}$. We show in Theorem \ref{upper_adaptive} that this is indeed the case: the minimax lower bounds of Theorem \ref{minimax_rate} can be matched up to logarithmic factors by some careful adaptation approach that essentially compares the performance of the empirical best source and target classifiers on the target data. This is described in Section \ref{Adaptive_Sec}.

\begin{figure}[t]
  \centering
\includegraphics[height=0.18\textwidth,width=0.67\textwidth]{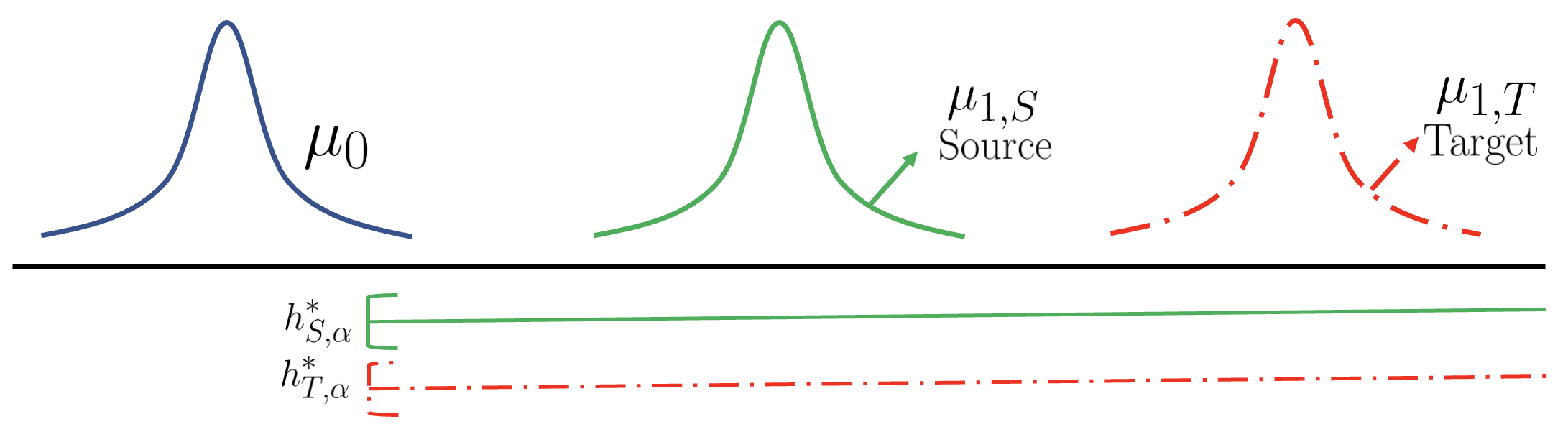}
  \caption{$\mu_0=\mathcal{N}(a_0,\sigma^2)$ is the common distribution, and $\mu_{1,S},\mu_{1,T}=\mathcal{N}(a_{1,S},\sigma^2$), $\mathcal{N}(a_{1,T},\sigma^2)$ are the source and target distributions. The universally optimal decision rules for the source and target are identical, i.e.,  $h^*_{S,\alpha}=h^*_{T,\alpha}$, in fact for any value of $\alpha \in [0, 1]$ (see Section \ref{full_trans}).} %due to Neyman-Pearson Lemma.}
  %correspond to the distributions of normal class, source abnormal class, and target abnormal class, respectively, and $t$ is the point such that $\int_{t}^{+\infty}\frac{1}{\sqrt{2\pi}}e^{\frac{-x^2}{2}}dx=\alpha$. The dashed line is the optimal decision rule (most powerful test) for both normal vs. source abnormal classes and normal vs. target abnormal classes.}
  \label{fig1}
\end{figure}

\subsection{Related Work}\label{related}
Outlier detection and transfer learning have mostly been studied separately, despite the clear need for transfer learning in applications of outlier detection where the rare class of data is by definition, always scarce. 

As such, transfer learning works have mostly focused on traditional classification and regression starting from seminal works of \cite{mansour2009domain,david2010impossibility,ben2010theory,ben2006analysis}, to more recent works of \cite{hanneke2019value,hanneke2022no,kalan2022statistical,mousavi2020minimax,lei2021near}. The works of \cite{hanneke2019value,hanneke2022no} are most closely related as our notion of outlier transfer exponent may be viewed as an extention of their notion of transfer exponent; however, besides for the fact that both notions capture discrepancies around decision boundaries, transfer in outlier detection is fundamentally different from the case of traditional classification studied in the above works: for instance, as stated earlier, distributions that are significantly different in traditional classification can be quite close in outlier transfer as revealed in this work.

Theoretical works on outlier detection on the other hand have mostly focused on unsupervised and supervised settings, but without considering the more practical transfer setting. Unsupervised outlier detection assumes that only data from the common class $\mu_0$ is available; theoretical works include studies of density-level set estimation \citep{steinwart2005classification, polonik1995measuring,ben1997learning,tsybakov1997nonparametric} where \emph{outliers} are viewed as data in low density regions, or in works on so-called \emph{one-class classification} that aim to learn a contour of the common class $\mu_0$ \citep{scholkopf2001estimating}. Supervised outlier-detection has commonly been formalized via Neyman-Pearson classification, where some data from both the common and rare classes are used to optimize and constrain empirical errors. Early works include \citep{cannon2002learning, scott2005neyman, blanchard2010semi, rigollet2011neyman} which establish convergence rates in various distributional and model selection settings, but all exclude the question of transfer. 

Transfer learning for outlier detection has in fact received much recent attention in the methodological literature \citep{xiao2015robust, andrews2016transfer, ide2017multi, chalapathy2018anomaly, yang2020anomaly} where various approaches have been proposed that aim to leverage shared structural aspects of source and target rare class data. 

On the theoretical front however, much less is understood about outlier transfer. The recent work of \cite{scott2019generalized} initiates theoretical understanding of the problem: they are first to show that, in some situations where both source and target share the same optimal classifiers, various procedures can guarantee consistency (i.e., taking sample size to infinity) even as source and target $\mu_{1, S}, \mu_{1, T}$ appear different. Our Proposition \ref{lem6} shows that in fact optimal classifiers may be shared in even more general situations, similarly implying consistency for seemingly very different source and target rare class distributions. Our main results of Theorems \ref{minimax_rate} and \ref{upper_adaptive}
reach further in deriving the first insights into the finite-sample regimes of outlier transfer, by establishing information-theoretic limits of the problem, along with a notion of discrepancy 
between source and target that tightly captures such limits.

\section{Setup}
\label{setup-paper}
We first formalize the Neyman-Pearson classification framework, followed by its extension to the transfer case. 
\subsection{Neyman-pearson Classification}
%In this section we define the source and target outlier detection problems via Neyman-Pearson framework and introduce an appropriate notion of distance between source and target.We first introduce a generic \emph{Neyman-Pearson} classification problem.
Let $\mu_0$ and $\mu_1$ denote probability distributions on some measurable space ($\mathcal{X},\Sigma)$. Furthermore, suppose that $\mathcal{H}$ is a hypothesis class consisting of measurable $0$-$1$ functions on the domain $\mathcal{X}$, where we view $h(x) =0$ or $1$ as predicting that $x$ is generated from class $\mu_0$ or $\mu_1$. We view $\mu_0$ and $\mu_1$ as representing a \emph{common} and \emph{rare} class of (future) data. 

\begin{definition} We are interested in the so-called \emph{Type-I} and \emph{Type-II} errors defined as follows: 
\begin{align*} 
R_{\mu_0}(h)=\mu_0(h(x)=1), \quad R_{\mu_1}(h)=\mu_1(h(x)=0). 
\end{align*} 

\end{definition}

\emph{Neyman-Pearson} classification then refers to the problem of minimizing Type-II error subject to low Type-I error: 
\begin{align}\label{eq0}
    &\underset{h\in \mathcal{H}}{\text{Minimize}} \ R_{\mu_1}(h)\nonumber\\
    & \text{s.t.} \ R_{\mu_0}(h)\leq \alpha
\end{align}

Under mild conditions, the \emph{universally} optimal classifier, i.e., taking $\cal H$ as the set of all measurable $0$-$1$ functions, is fully characterized by the classical Neyman-Pearson Lemma (see Appendix \ref{APP_A}) in terms of \emph{density ratios}. Namely, let $p_0$ and $p_1$ denote densities of $\mu_0$ and $\mu_1$ w.r.t. some dominating measure $\nu$, then the minimizer of (\ref{eq0}) has the form $h^*_\alpha(x) = \mathbb{1}_{\{ \frac{p_1(x)}{p_0(x)} \geq \lambda\}}$ whenever there exists $\lambda$ such that $R_{\mu_0} (h_{\alpha}^*)$ is exactly $\alpha$\footnote{If we further allow for randomized classifiers, then Neyman Pearson Lemma fully characterizes universally optimal solutions of (\ref{eq0}) and establishes uniqueness almost-surely under mild restrictions.}. 

In Section \ref{full_trans} we ask when such universal minimizer transfers across source and target rare class distributions.
\subsection{Transfer Learning Setup}
{\bf Population Setup.} We consider the following two source and target Neyman-Pearson problems, defined for a fixed common class distribution $\mu_0$, and source and target rare class distributions $\mu_{1,S}$ and $\mu_{1,T}$:

\begin{minipage}{.35\linewidth}
  \centering
  \begin{align}\label{eq1}
    &\underset{h\in \mathcal{H}}{\text{Minimize}} \ R_{\mu_{1,S}}(h)\nonumber\\
    & \text{s.t.} \ R_{\mu_0}(h)\leq \alpha
  \end{align}
\end{minipage}%
\begin{minipage}{.2\linewidth}
\centering \
\end{minipage}%
\begin{minipage}{.35\linewidth}
  \centering
  \begin{align}\label{eq2}
    &\underset{h\in \mathcal{H}}{\text{Minimize}} \ R_{\mu_{1,T}}(h)\nonumber\\
    & \text{s.t.} \ R_{\mu_0}(h)\leq \alpha
  \end{align}
\end{minipage}

We let $(\mu_0,\mu_{1,S}, \alpha)$ and $(\mu_0,\mu_{1,T}, \alpha)$ denote these source and target problems. We will see later that the limits of outlier transfer, especially in finite-sample regimes, are well captured by discrepancies between these two problems. In particular, we will be interested in discrepancies between optimal solutions and the measure of the corresponding decision boundaries under involved distributions. We henceforth let $h^*_{S, \alpha}$ and $h^*_{T, \alpha}$ denote (not necessarily unique) {\bf solutions} of (\ref{eq1}) and (\ref{eq2}) (which we assume exist).

{\bf Finite-Sample Setup.} We assume access to $n_0, n_S, n_T$ i.i.d. data points respectively from $\mu_0, \mu_{1, S}, \mu_{1, T}$, where we allow $n_T = 0$. The transfer-learning procedure is then allowed to return $\hat h \in \cal H$ satisfying 
$$R_{\mu_0}(\hat h) \leq \alpha + \epsilon_0, $$
for some slack $\epsilon_0 = \epsilon_0(n_0)$, usually of order $n_0^{-1/2}$. The goal of the learner is to minimize the {\bf target-excess error}

$$\mathcal{E}_{1,T}(\hat{h})\doteq\max\big\{0,R_{\mu_{1,T}}(\hat{h})-R_{\mu_{1,T}}(h^*_{T,\alpha})\big\}.$$

A main aim of this work is to understand which rates of $\mathcal{E}_{1,T}(\hat{h})$ are achievable in terms of sample sizes $n_S$ and $n_T$, and which notions of discrepancy from source to target help capture these limits.

%The optimal solutions of source and target problems may not be unique. We therefore need to consider the set of optimal solutions.
%\begin{definition}$S^*_{\alpha},T^*_{\alpha}\subset \mathcal{H}$  denote the {\bf sets of optimal solutions} of (\ref{eq1}) and (\ref{eq2}), respectively.
%\end{definition}

\section{Equivalence of Population Problems}\label{full_trans}
As discussed in the introduction, we may have seemingly very different source and target distributions $\mu_{1,S}$ and $\mu_{1, T}$ which however yield the same (universally) optimal classifiers. We investigate this phenomena in this section, the main aim being to illustrate how fundamentally different outlier transfer is from traditional classification. To this end, we consider the set $\cal U$ of all possible measurable $0$-$1$ functions on $\cal X$, and let $\cal H = \cal U$ in the dicussion below. 

We will henceforth say that the source problem $(\mu_0, \mu_{1, S}, \alpha)$ is {\bf equivalent} to the target problem $(\mu_0, \mu_{1, T}, \alpha)$ (at the population level), if all solutions to (\ref{eq1}) are also solutions to (\ref{eq2}). Clearly, when this is the case, source samples alone can drive down the target risk at least asymptotically. 

We first notice that Neyman-Pearson Lemma offers some immediate answers under mild conditions. To see this, let $p_0, p_{1, S}, p_{1, T}$ denote densities w.r.t. a common dominating measure $\nu$. In what follows we will simply let the dominating measure $\nu \doteq  \mu_0 + \mu_{1, S} + \mu_{1, T}$. As previously discussed, Neyman-Pearson Lemma characterizes optimal classifiers in terms of level sets of density ratios. 

\begin{definition}[Level sets]\label{sub}
$\mathcal{L}_{\lambda}^{S}:=\{x: \frac{p_{1,S}(x)}{p_0(x)}\geq \lambda\}$ and $\mathcal{L}_{\lambda}^{T}:=\{x: \frac{p_{1,T}(x)}{p_0(x)}\geq \lambda\}$. Here, when the denominator of a fraction is zero we view it as infinity no matter if the nominator is zero or nonzero.
\end{definition}

The following then establishes \emph{equivalence} between source and target under mild conditions as a direct consequence of Neyman-Pearson Lemma. 

\begin{proposition}\label{prop:NP-Corollary}
Suppose $\mu_0, \mu_{1, S}, \mu_{1, T}$ are mutually dominating. Assume further that  
$\mu_0(\mathcal{L}_{\lambda}^{S})$ is continuous and strictly monotonic in $(0, 1)$ as a function of $\lambda$. Then, if $\{\mathcal{L}_{\lambda}^{S}\}_{\lambda\geq 0} \subset  \{\mathcal{L}_{\lambda}^{T}\}_{\lambda\geq 0}$, we have for all $0 < \alpha < 1$, that any $h^*_{S, \alpha}$ is a solution of the target problem (\ref{eq2}). 
% For fixed $0 < \alpha < 1$, suppose the Neyman-Pearson source classifier 
% $h^*_{S, \alpha} \doteq \mathbb{1}\{\frac{p_{1,S}}{p_{0, S}} \geq \lambda \}$ satisfies $R_{\mu_0}(h^*_{S, \alpha}) = \alpha$. Then the source problem is equivalent to the target problem whenever 
\end{proposition}

The above statement was already evident in the work of \cite{scott2019generalized} where it is assumed that source density ratios are given as a monotonic function of the target density ratios; this immediately implies $\{\mathcal{L}_{\lambda}^{S}\} \subset  \{\mathcal{L}_{\lambda}^{T}\}$. 

The statement is illustrated in the example of Fig \ref{fig1} with Gaussian distributions where the source may appear significantly different from the target (and in fact would yield different Bayes classifiers in traditional classification). To contrast, consider the example of Fig \ref{fig2} where the source problem yields different level sets than for the target problem (simply by changing the Gaussian variance), and we hence do not have equivalence. 

\begin{remark}[Issue with Disjoint Supports]
We have assumed in Proposition \ref{prop:NP-Corollary} that all 3 distributions are mutually dominating, while in practice this might rarely be the case. However, without this assumption (essentially of shared support), we may not easily have equivalence between source and target. 

For intuition, consider a region $A$ of space where $\mu_{1, S}(A) = 0$ while $\mu_{1, T}(A) > 0$. Then let $h^{*, 0}_{S, \alpha} = 0$ on $A$ and $h^{*, 1}_{S, \alpha} = 1$ both optimal under the source problem; 
clearly we may have $R_{1, T}(h^{*, 0}_{S, \alpha}) >  R_{1, T}(h^{*, 1}_{S, \alpha})$ since $\mu_{1, T}(A) > 0$. 
\end{remark} 

The rest of this section is devoted to handling this issue by restricting the attention to more \emph{reasonable} classifiers that essentially classify any point outside the support of $\mu_0$ as $1$. The right notion of support is critical in establishing our main relaxation of the above proposition, namely Proposition \ref{lem6} below. 

We require the following definitions and assumptions. 

\begin{definition}[Restricted Universal Hyposthesis Class]
We let $\mathcal{U}^*$ consist of all $0$-$1$ measurable functions on the domain $\mathcal{X}$ such that for every $h\in\mathcal{U}^*$, $h\equiv 1$ on $\{x:p_0(x)=0\}$ a.s. $\nu$.
\end{definition}

\begin{figure}[t]
  \centering
\includegraphics[height=0.18\textwidth,width=0.67\textwidth]{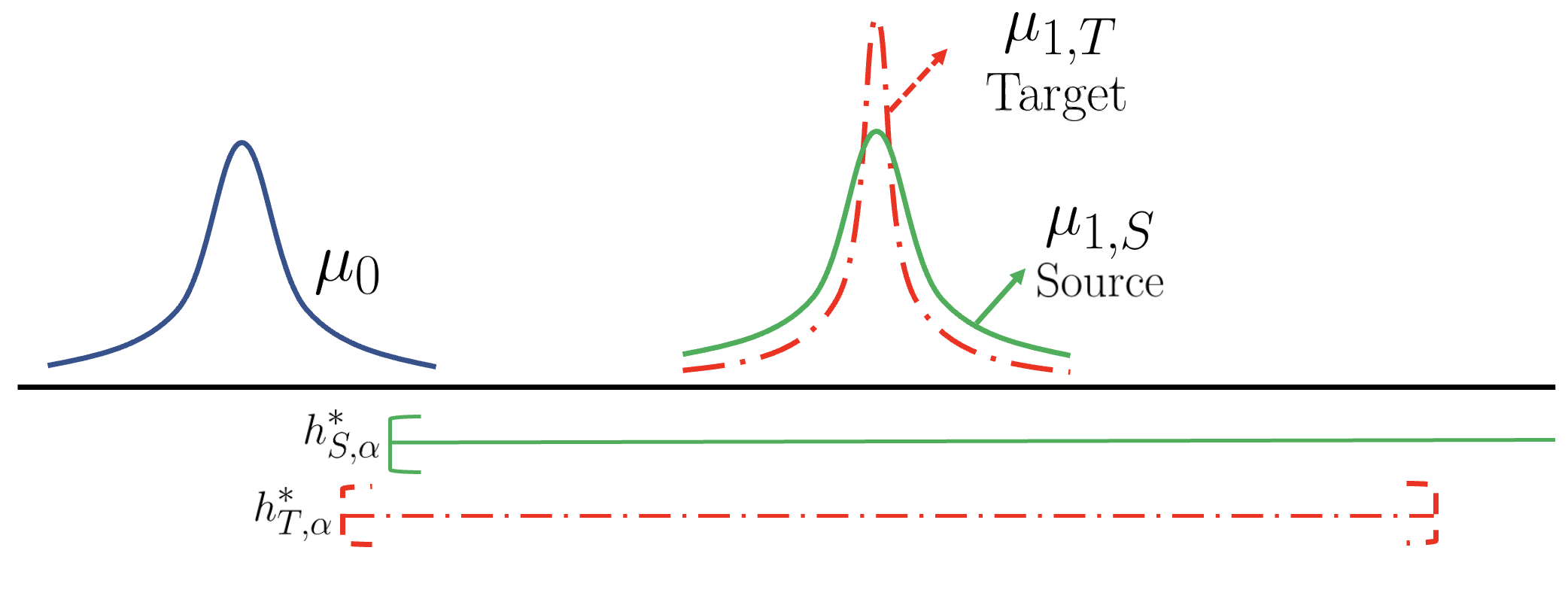}
  \caption{$\mu_0,\mu_{1,S},\mu_{1,T}=\mathcal{N}(a_0,\sigma^2),\mathcal{N}(a_{1,S},\sigma^2),\mathcal{N}(a_{1,T},\sigma'^2)$ where $a_{1,S}=a_{1,T}$ and $\sigma'<\sigma$. Optimal decision rules differ.} % of the source is $h^*_{S,\alpha}=\mathbb{1}_{\mathcal{L}^S_{\lambda}}$ where $\mathcal{L}^S_{\lambda}=\{x:\frac{p_{1,S}(x)}{p_0(x)}\geq \lambda\}$ and $\mu_0(\mathcal{L}^S_{\lambda})=\alpha$. While the optimal decision rule of target is $h^*_{T,\alpha}=\mathbb{1}_{\mathcal{L}^T_{\lambda'}}$ where $\mathcal{L}^T_{\lambda'}=\{x:\frac{p_{1,T}(x)}{p_0(x)}\geq\lambda'\}$ and $\mu_0(\mathcal{L}^T_{\lambda'})=\alpha.$} %due to Neyman-Pearson Lemma.}
  %correspond to the distributions of normal class, source abnormal class, and target abnormal class, respectively, and $t$ is the point such that $\int_{t}^{+\infty}\frac{1}{\sqrt{2\pi}}e^{\frac{-x^2}{2}}dx=\alpha$. The dashed line is the optimal decision rule (most powerful test) for both normal vs. source abnormal classes and normal vs. target abnormal classes.}
  \label{fig2}
\end{figure}

% Fig \ref{fig1} and Fig \ref{fig2} illustrate two pairs of source and target such that in one of them they share the same optimal decision rules and in the other they have different optimal decision rules. In Fig \ref{fig1} we have  $\mu_0,\mu_{1,S},\mu_{1,T}=\mathcal{N}(a_0,\sigma^2),\mathcal{N}(a_{1,S},\sigma^2),\mathcal{N}(a_{1,T},\sigma^2)$ where $a_{1,T}>a_{1,S}>a_0$. Due to Neyman-Pearson Lemma, the optimal solutions for the source and target problem are $h^*_{S,\alpha}=h^*_{T,\alpha}=\mathbb{1}_{\mathcal{L}^S_{\lambda}}$ where $\mathcal{L}^S_{\lambda}=\{x:\frac{p_{1,S}(x)}{p_0(x)}\geq\lambda\}$ and $\mu_0(\mathcal{L}^S_{\lambda})=\alpha$. On the other hand, in Fig \ref{fig2} we have  $\mu_0,\mu_{1,S},\mu_{1,T}=\mathcal{N}(a_0,\sigma^2),\mathcal{N}(a_{1,S},\sigma^2),\mathcal{N}(a_{1,T},\sigma'^2)$ , where $a_{1,S}=a_{1,T}$ and $\sigma'<\sigma$, which is an example of non-full transferable pair of source and target. The example depicted in Fig \ref{fig2} suggests that a slight difference in the variance of two Gaussian distributions with the same mean results in different optimal solutions for the source and target.

%We make the following assumption on the level sets.
\begin{definition}\label{assum}
 We say that {\bf $\alpha$ is achievable} if there exist thresholds $\lambda$ and $\lambda'$ such that $\mu_0(\mathcal{L}_{\lambda}^{S})=\alpha$ and $\mu_0(\mathcal{L}_{\lambda'}^{T})=\alpha$.
 %with $\nu(\{x:p_0(x)\neq 0, \frac{p_S(x)}{p_0(x)}=\lambda\})=\nu(\{x:p_0(x)\neq 0, \frac{p_T(x)}{p_0(x)}=\lambda'\})=0$
\end{definition}
\begin{definition}[$\alpha$-level set]\label{alpha-level} Whenever $\alpha$ is achievable, we define $\mathcal{L}^S(\alpha)$ as the level set in the source whose measure under $\mu_0$ is $\alpha$. The definition of $\mathcal{L}^T(\alpha)$ which corresponds to the target is the same.
\end{definition}
\begin{remark}\label{rem1.3}
Definition \ref{assum} ensures that $\mathcal{L}^S(\alpha)$ exists, but it may not be unique. However, we will show in Appendix \ref{APP_A} by Proposition \ref{pro_ney}, it is unique a.s. $\nu$.
\end{remark}

% Direct application
% of Neyman-Pearson Lemma \ref{thm1}, implies the following proposition. 
% \begin{proposition}\label{prop_ney_pearson}
%     Let $\mu_0,\mu_{1,S},\mu_{1,T}$ be such that for all $\lambda>0$,
%     \begin{align*}
%     \{x:p_0(x)=\lambda p_{1,S}(x)\}\ \ \text{and}\ \ \{x:p_0(x)=\lambda p_{1,T}(x)\}
% \end{align*}
% have zero measure under $\nu$. Then every solution of (\ref{eq1}) in $\mathcal{U}$ is also a solution of (\ref{eq2}).
% \end{proposition}

% Following example depicted in Fig \ref{fig2_example} shows that if the distributions do not satisfy the condition of Proposition \ref{prop_ney_pearson}, then the result no longer holds. Specifically, it
 %illustrates the case that even if $\mathcal{L}^S(\alpha)=\mathcal{L}^T(\alpha)$, there could be a function $h$ such that $h$ is an optimal solution for the source but not for the target.

%Neyman-Pearson Lemma does not characterize the tranferability of solutions from source to target for all distributions. In addressing this issue, we only consider the classifiers in $\mathcal{U}^*$ rather than $\mathcal{U}$ and 
%define the full transferability as follows:
% \begin{definition}\label{new_full}
%     We say $(\mu_0,\mu_{1,S},\alpha)$ fully transfers to $(\mu_0,\mu_{1,T},\alpha)$ if every
% solution of (\ref{eq1}) in $\mathcal{U}^*$ is also a solution of (\ref{eq2}).
% \end{definition}

The following proposition relaxes the conditions of Proposition \ref{prop:NP-Corollary} above by restricting attention to universal classifiers in $\cal U^*$. Its proof is technical to various corner cases described in Appendix \ref{APP_A}. 

\begin{proposition}\label{lem6}
Let $0\leq\alpha<1$ and suppose that $\alpha$ is achievable. Then $(\mu_0,\mu_{1,S}, \alpha)$ is equivalent to $(\mu_0,\mu_{1,T}, \alpha)$ under $\mathcal{U}^*$ iff $\mathcal{L}^S(\alpha)\in \{\mathcal{L}_{\lambda}^{T}\}_{\lambda\geq 0}$ a.s. $\nu$. In particular, if $\alpha$ is achievable for all $0\leq \alpha<1$ and $\mathcal{L}^S(\alpha)\in \{\mathcal{L}_{\lambda}^{T}\}_{\lambda\geq 0}$ a.s. $\nu$ for all $0\leq \alpha<1$, then $(\mu_0,\mu_{1,S}, \alpha)$ is equivalent to $(\mu_0,\mu_{1,T}, \alpha)$ for all $0\leq \alpha<1$.
\end{proposition}
\begin{proof}
    See \ref{proof_suff}.
\end{proof}
\begin{remark}\label{new_remark}
Notice that the statements of Proposition \ref{lem6} trivially also hold over any hypothesis class $\cal H$ (rather than just $\cal U^*$) containing the level sets $\mathbb{1}_{\mathcal{L}^S(\alpha)},\mathbb{1}_{\mathcal{L}^T(\alpha)}\in \mathcal{H}$ and where, for every $h\in\mathcal{H}$, $h\equiv 1$ on $\{x:p_0(x)=0\}$ a.s. $\nu$.
\end{remark}

We illustrate this final proposition in Figure \ref{fig2_example}: the simple restriction to $\cal U^*$ reveals more scenarios where source is equivalent to target (at the population level) even when the distributions appear significantly different. 

\begin{figure}
\centering
\begin{minipage}{.23\textwidth}
\includegraphics[height=0.35\textwidth,width=1.3\textwidth]{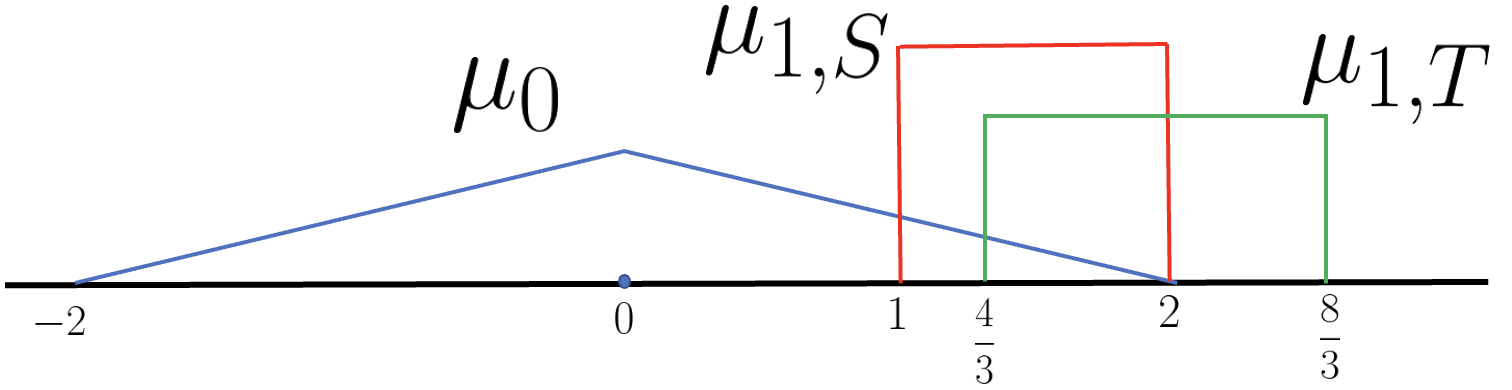}
  \subcaption{$\mu_0,\mu_{1,S},\mu_{1,T}$}\label{sub1_example}
\end{minipage}
\hspace{10mm}
%\hfill
\begin{minipage}{.23\textwidth}
  \includegraphics[height=0.35\textwidth,width=1.3\textwidth]{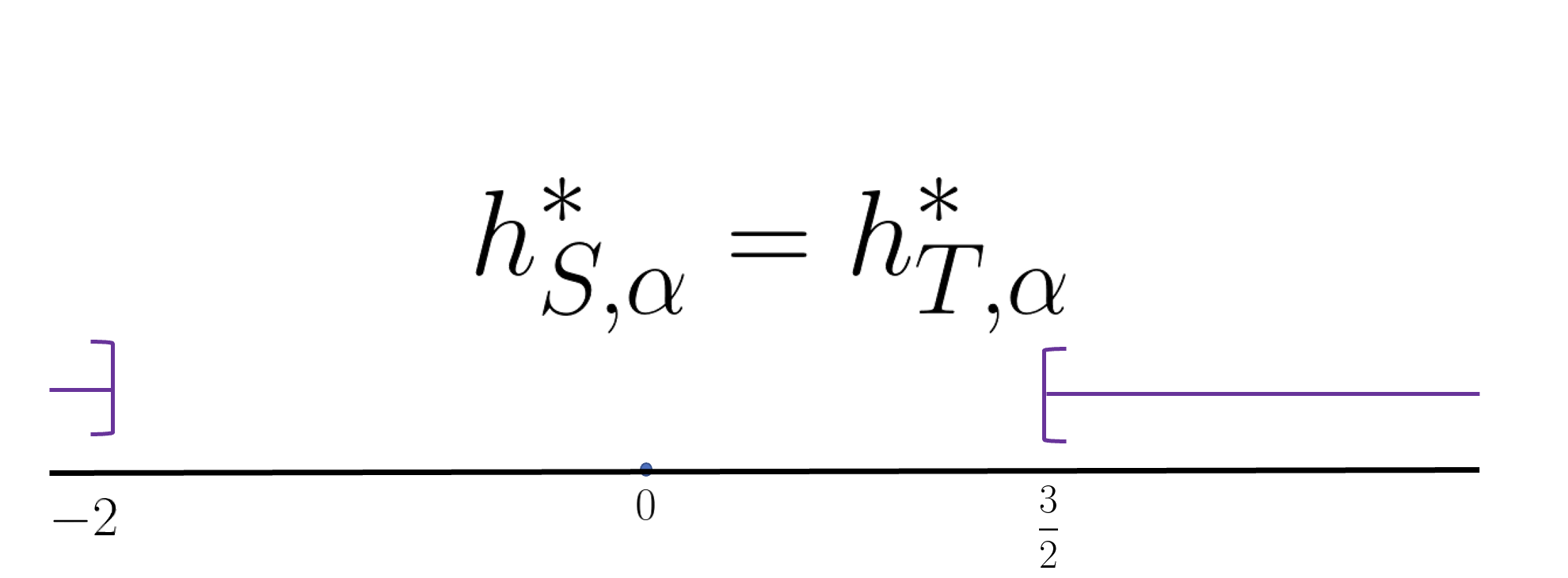}

  \subcaption{Source and target solutions. }\label{sub2_example}
\end{minipage}
\hspace{10mm}
%\hfill
\begin{minipage}{.23\textwidth}
  \includegraphics[height=0.42\textwidth,width=1.1\textwidth]{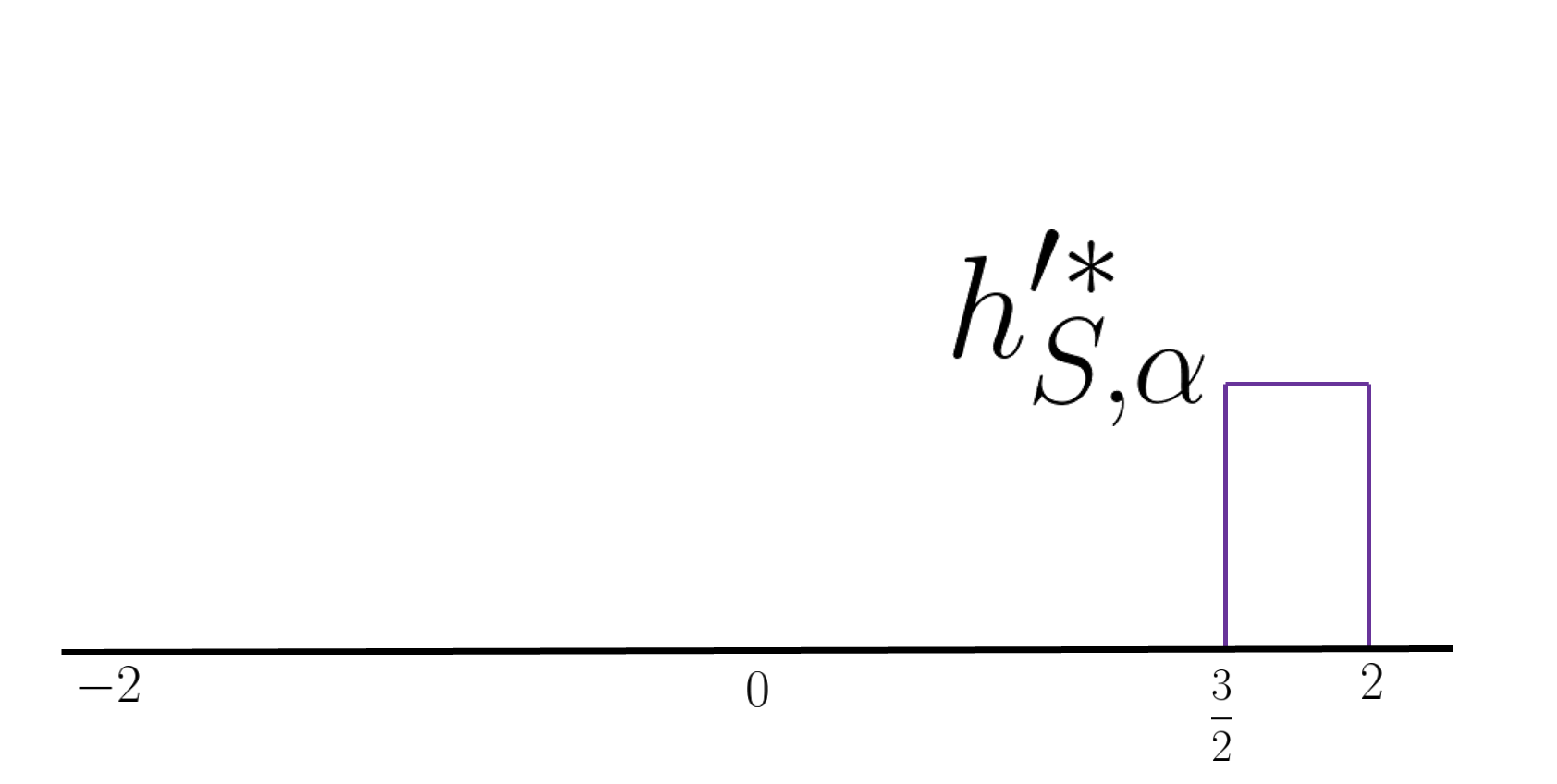}
  \subcaption{$h'^*_{S,\alpha}$ is a solution for source but not for target.}\label{sub3_example}
\end{minipage}
\caption{Illustration of Example \ref{ex_3.4} (see Appendix \ref{App_3}), where $h^*_{S,\alpha}=h^*_{T,\alpha}$ in $\mathcal{U}^*$, while $h'^*_{S,\alpha} \in \mathcal{U}$ is also a solution for source but not for target. In other words, the source problem remains equivalent to the target over the more \emph{reasonable} decision rules of $\cal U^*$.}
%corresponding to the level set for $\alpha=\frac{1}{16}$, but there exists an $h'^*_{S,\alpha}$ such that $h'^*_{S,\alpha}$ is an optimal solution for the source but not for the target.}
\label{fig2_example}%
\end{figure}
\section{Finite-Sample Results}
Neyman-Pearson Lemma offers the solution(s) of the optimization problem (\ref{eq2}) when we have the knowledge of the underlying distributions. However, in practical scenarios, we typically lack information about these distributions and only possess some samples drawn from them. In addressing this challenge, \cite{cannon2002learning} embarked on an empirical study of Neyman-Pearson classification and introduced a relaxed version of Neyman-Pearson classification problem. Let $n_0$ and $n_T$ be the number of i.i.d. samples
generated by $\mu_0$ and $\mu_{1,T}$, respectively, and $\epsilon_0 >0$. \cite{cannon2002learning} proposed the following optimization problem:
\begin{align}\label{eq3}
    \hat{h}=&\argmin_{h\in \mathcal{H}} \ \hat{R}_{\mu_{1,T}}(h)\nonumber\\
    & \text{s.t.} \ \hat{R}_{\mu_0}(h)\leq \alpha+\frac{\epsilon_0}{2}
\end{align}
where $\hat{R}_{\mu_0}(h)=\frac{1}{n_0}\sum\limits_{X_i \sim \mu_0} \mathbb{1}_{\{h(X_i)\neq 0\}}$ and $\hat{R}_{\mu_{1,T}}(h)=\frac{1}{n_T}\sum\limits_{X_i \sim \mu_{1,T}} \mathbb{1}_{\{h(X_i)\neq 1\}}$, and then derived the convergence rate of excess error in terms of the number of samples and VC dimension of the hypothesis class.

Neyman-Pearson classification in the setting of transfer finite-sample scenarios remains unexplored. In this section, Our objective is to understand the fundamental limits of transfer outlier detection in the finite-sample regime, where there are $n_0,n_S,n_T$ i.i.d. samples from $\mu_0,\mu_{1,S},\mu_{1,T}$, under a measure of discrepancy between source and target. We first define a discrepancy measure in transfer outlier detection and then characterize the fundamental limits of the problem by deriving a minimax lower bound in terms of the number of samples as well as the notion of discrepancy. Finally, we show that this lower bound is achievable through an adaptive algorithm which does not require the prior knowledge of the discrepancy between source and target.

\subsection{Outlier Transfer Exponent}
We aim to define an appropriate notion of outlier transfer distance between source and target under a hypothesis class. Here we adapt the transfer exponent notion defined in \cite{hanneke2019value} to a notion of discrepancy between source and target in outlier detection.
\begin{definition}[Outlier transfer exponent]\label{def_dis}
    Let $S^*_{\alpha}\subset\mathcal{H}$ denote the set of solutions of source (\ref{eq1}). Furthermore, let $(\mu_0,\mu_{1,S}, \alpha)$ and $(\mu_0,\mu_{1,T}, \alpha)$ denote the source and target problems, respectively. We call $\rho(r)>0$ outlier transfer exponent from $(\mu_0,\mu_{1,S}, \alpha)$ to $(\mu_0,\mu_{1,T}, \alpha)$ under $\cal H$, if there exist $r, C_{\rho(r)}>0$ such that 
\begin{align}\label{dist}
        C_{\rho(r)}\cdot \max\bigg\{0,R_{\mu_{1,S}}(h)-R_{\mu_{1,S}}(h^*_{S,\alpha})\bigg\}\geq \max\bigg\{0,R_{\mu_{1,T}}(h)-R_{\mu_{1,T}}(h^*_{S,\alpha})\bigg\}^{\rho(r)} 
    \end{align}
for all $h \in \mathcal{H} \ \text{with} \ R_{\mu_0}(h)\leq \alpha+r$, where $h^*_{S,\alpha}=\underset{h\in S^*_{\alpha}}{\argmax} \  R_{\mu_{1,T}}(h)$.
\end{definition}
We will show later (Theorem \ref{minimax_rate}) that this notion of discrepancy captures the fundamental limits of transfer outlier detection. The following example shows that for a fixed target and for any arbitrary $\rho\geq 1$, there exists a source such that it shares the same optimal decision rule as the target and attains the outlier transfer exponent $\rho$ with coefficient $C_{\rho}=1$. Therefore, for a given target, there can exist a wide range of sources with differing levels of discrepancy, all of which nevertheless share the same optimal decision rules.
%In the following example, there are one target and 
%two sources with the same transfer exponents, $\rho(r)$, and different coefficients, $C_{\rho(r)}$. 
%\begin{example}[Same transfer exponents]
%Let $\mu_0\sim \mathcal{N}(0,1), \mu_{1,T} \sim Unif[0,1], \mu_{1,{S_1}}\sim Unif[0.2,1], \mu_{1,{S_2}}\sim  Unif[0.4,1]$ and $\mathcal{H}=\{\mathbb{1}_{\{t\leq x\leq 1\}}(x): t\in[0,1]\}$. Furthermore, let $\alpha$ be sufficiently small so that $\alpha=\mu_0([a,1])$ for some $a\in [0.4,1]$. Then by Neyman-Pearson Lemma the optimal solutions of the source and target are the same and unique $h^{*}_{S_1,\alpha}=h^{*}_{S_2,\alpha}=h^{*}_{T,\alpha}=\mathbb{1}_{\{a \leq x\leq 1\}}$. Then for any $r\in[0,\alpha]$, we have $(C_{\rho},\rho)=(0.8,1)$ for $(\mu_0,\mu_{1,{S_1}},\alpha)$ to $(\mu_0,\mu_{1,T},\alpha)$ and $(C_{\rho},\rho)=(0.6,1)$ for $(\mu_0,\mu_{1,{S_2}},\alpha)$ to $(\mu_0,\mu_{1,T},\alpha)$. Therefore source2 is closer to the target since it achieves a smaller coefficient $C_{\rho}$.
%\end{example}
%In the upcoming example, there are two pairs of source and target, with both pairs sharing the same target. In both of the pairs, the optimal decision rules of the source and target are the same, albeit with different transfer exponents.
\begin{example} Let $\mu_0\sim \mathcal{N}(-1,1)$, $\mu_{1,T}\sim Unif[0,1]$, $p_{1,S}=\rho x^{\rho-1}\mathbb{1}_{\{x\in[0,1]\}}$ for $\rho\geq 1$ where $p_{1,S}$ is the density of $\mu_{1,S}$ w.r.t. Lebesgue measure. Furthermore, let $\alpha=\mu_0([0,1])$ ,$\mathcal{H}=\{\mathbb{1}_{\{t\leq x\leq 1\}}(x):t\in [-1,1]\}$, and $r$ be small enough. Then, we have $h^*_{T,\alpha}=h^*_{S,\alpha}=\mathbb{1}_{\{0\leq x\leq 1\}}$. Moreover, for $h=\mathbb{1}_{\{t\leq x\leq 1\}}$ for $t\geq 0$ we obtain 
\begin{align*}
R_{\mu_{1,T}}(h)-R_{\mu_{1,T}}(h^*_{S,\alpha})=t \ \ \text{and}\ \ 
R_{\mu_{1,S}}(h)-R_{\mu_{1,S}}(h^*_{S,\alpha})=t^{\rho}.
\end{align*}
Hence, the outlier transfer exponent is $\rho$ and the coefficient $C_{\rho}$ is $1$. 
\end{example}
Following proposition shows the effect of $r$ on the outlier transfer exponent $\rho(r)$. For small enough $r$, $\rho(r)$ could be small while for a large $r$, $\rho(r)$ is infinite.
\begin{proposition}\label{prop_example}
    There exist $\mu_0,\mu_{1,S},\mu_{1,T},\mathcal{H},\alpha,r$ such that for any $h\in \mathcal{H}$ with $R_{\mu_0}(h)\leq \alpha+r$, (\ref{dist}) holds for $\rho(r)=1$ and there exists an $h\in\mathcal{H}$ with $R_{\mu_0}(h)> \alpha+r$ for which (\ref{dist}) does not hold for any $\rho<\infty$.
\end{proposition}

\subsection{Minimax Lower Bound for Transfer Outlier Detection}\label{sec5}

Equipped with the notion of outlier transfer exponent, we characterize the fundamental limits of transfer outlier detection by establishing a minimax lower bound. To achieve this objective, first we need to specify the class of distributions for which the minimax lower bound is derived.
\begin{definition}[Class of distributions] 
Fix a hypothesis class $\mathcal{H}$ with finite VC dimension $d_{\mathcal{H}}$. Let $\mathcal{F}_{\mathcal{H}}(\rho,\alpha,C,\Delta)$ denote the class of triplets of distributions $(\mu_0,\mu_{1,S},\mu_{1,T})$ for which $\alpha$ is achievable according to Definition \ref{assum}, $\mathcal{E}_{1,T}(h^*_{S,\alpha})\leq \Delta$, and there exist $\rho(r)\leq \rho, C_{\rho(r)}\leq C$ for any $0<r<\frac{2\alpha}{d_{\mathcal{H}}}$ per Definition \ref{def_dis}.
\end{definition}
\begin{remark}
    Deriving a minimax lower bound for the class of distributions satisfying $\alpha$ achievability is a stronger result than for the class without necessarily satisfying that, as the former is contained in the latter.
\end{remark}

\begin{theorem}[Minimax lower bound]\label{minimax_rate}
Fix a hypothesis class $\mathcal{H}$ with finite VC dimension $d_{\mathcal{H}}\geq 3$. Moreover, let $\alpha<\frac{1}{2}, \rho\geq 1, \Delta\leq 1$, and $\delta_0>0$. Furthermore, Let $\hat{h}$ be any learner's classifier that is supposed to output a classifier from $\{h\in \mathcal{H}: \mu_0(h)\leq \alpha+\frac{2\alpha}{d_{\mathcal{H}}}\}$ with probability at least $1-\delta_0$, by having access to $n_0,n_S,n_T$ i.i.d. samples from $\mu_0,\mu_{1,S},\mu_{1,T}$, respectively. Denote the samples from $\mu_0,\mu_{1,S},\mu_{1,T}$ by $S_{\mu_0},S_{\mu_{1,S}},S_{\mu_{1,T}}$ and suppose that there are sufficiently many samples such that $\min\{\Delta+(\frac{d_{\mathcal{H}}}{n_S})^{\frac{1}{2\rho}},(\frac{d_{\mathcal{H}}}{n_T})^{\frac{1}{2}}\}\leq 2$. Then, we have
\begin{align}\label{lower_label_ineq}
\inf_{\hat{h}}\sup_{\mathcal{F}(\rho,\alpha,1,\Delta)}\underset{S_{\mu_0},S_{\mu_{1,S}},S_{\mu_{1,T}}}{\mathbb{E}}\big[\mathcal{E}_{1,T}(\hat{h})\big]\geq c\cdot \min\{\Delta+(\frac{d_{\mathcal{H}}}{n_S})^{\frac{1}{2\rho}},(\frac{d_{\mathcal{H}}}{n_T})^{\frac{1}{2}}\}
\end{align}
where $c$ is a numerical constant.
\end{theorem}
\begin{remark}
In Theorem \ref{minimax_rate}, the minimax lower bound for the class of equivalent source and target distributions, i.e, $\Delta=0$, reduces to $c\cdot\min\{(\frac{d_{\mathcal{H}}}{n_S})^{\frac{1}{2\rho}},(\frac{d_{\mathcal{H}}}{n_T})^{\frac{1}{2}}\}$. In this case, by only having access to unlimited source samples, achieving an arbitrary small target-excess error is possible.
%if $(\mu_0,\mu_{1,S},\alpha)$ is equivalent to $(\mu_0,\mu_{1,T},\alpha)$ under $\mathcal{H}$ (see Remark \ref{new_remark}), then $\mathcal{E}_{1,T}(h^*_{S,\alpha})=0$ and the upper bound on $\mathcal{E}_{1,T}(\hat{h})$ reduces to $c\cdot\min\{(\frac{d_{\mathcal{H}}}{n_S})^{\frac{1}{2\rho}},(\frac{d_{\mathcal{H}}}{n_T})^{\frac{1}{2}}\}$. In this case, by only having access to unlimited source samples, achieving an arbitrary small target-excess error is possible.
\end{remark}
\begin{remark}
The outlier transfer exponent in the term $(\frac{d_{\mathcal{H}}}{n_S})^{\frac{1}{2\rho}}$ captures the relative effectiveness of source samples in the target. If source and target share the same optimal decision rules and $\rho=1$, source samples would be equally effective as target samples. However, even if the source and target share the same optimal decision rules, source samples may result in poor transfer performance when $\rho$ is large.
\end{remark}
\begin{remark}
    In Theorem \ref{minimax_rate}, the learner is allowed to output a classifier $\hat{h}$ with a Type-I error that slightly exceeds the pre-defined threshold $\alpha$. However, in certain applications, it is imperative to uphold the threshold without any exceeding. The minimax lower bound in Theorem \ref{minimax_rate}, implies that (\ref{lower_label_ineq}) holds even if the learner is only allowed to output a classifier $\hat{h}$ from $\{h\in \mathcal{H}: \mu_0(h)\leq \alpha\}$ with probability $1-\delta_0$.
\end{remark}
\subsection{Adaptive Rates}\label{Adaptive_Sec}
In Section \ref{sec5}, we establish a minimax lower bound that can be attained by an oracle that effectively disregards the least informative dataset from either the source or target. Let 
\begin{align*}
    &\hat{\mathcal{H}}_0=\{h\in \mathcal{H}: \hat{R}_{\mu_0}(h)\leq \alpha+\frac{\epsilon_0}{2}\}\\
    &\hat{h}_T=\underset{h\in \hat{\mathcal{H}}_0}{\argmin}  \ \hat{R}_{\mu_{1,T}}(h)\\
&\hat{h}_S=\underset{h\in \hat{\mathcal{H}}_0}{\argmin}\  \hat{R}_{\mu_{1,S}}(h).
\end{align*}
Then 
\begin{align*}
   \mathcal{E}_{1,T}(\hat{h}_S)\leq \mathcal{E}_{1,T}(h^*_{S,\alpha})+(\frac{d_{\mathcal{H}}}{n_S})^{\frac{1}{2\rho}} \ \ \ \text{and}\ \ \ \mathcal{E}_{1,T}(\hat{h}_T)\leq(\frac{d_{\mathcal{H}}}{n_T})^{\frac{1}{2}}
\end{align*}
with high probability. 
However, deciding whether $\hat{h}_S$ or $\hat{h}_T$ achieves a
smaller error requires the knowledge of outlier transfer exponent and the target Type-II error of the optimal source decision rule, which are not available in practical scenarios.

In this section, we show that by using an adaptive algorithm that takes source and target samples as input and produces a hypothesis $\hat{h}\in \mathcal{H}$ without using any additional information such as prior knowledge of the outlier transfer exponent, the rate is achievable. To accomplish this, we adapt the algorithm introduced in \cite{hanneke2019value}. First we state the following lemma proved by \cite{vapnik1974theory}.
\begin{lemma}\label{lem_vc}
    Let $\mu_1$ be a probability measure. Moreover,
    let $\mathcal{H}$ be a hypothesis class with finite VC dimension $d_{\mathcal{H}}$ and  $A_n=\frac{d_{\mathcal{H}}}{n}\log{\big(\frac{\text{max}\{n,d_{\mathcal{H}}\}}{d_{\mathcal{H}}}\big)}+\frac{1}{n}\log{(\frac{1}{\delta})}$. Then with probability at least $1-\frac{\delta}{3}$, $\forall h,h'\in \mathcal{H}$
    \begin{align*}
        R_{\mu_1}(h)-R_{\mu_1}(h')\leq \hat{R}_{\mu_1}(h)-\hat{R}_{\mu_1}(h')
        +c\sqrt{\min\{\mu_1(h\neq h'),\hat{\mu}_1(h\neq h')\}A_n}+cA_n
    \end{align*}
where $c\in(0,\infty)$ is a universal constant, $\hat{R}_{\mu_1}(h)=\frac{1}{n}\sum\limits_{X_i\sim \mu_1}\mathbb{1}_{\{h(X_i)\neq 1\}}$ for $n$ number of i.i.d. samples generated by $\mu_1$, and $\hat{\mu}_1$ denotes the corresponding empirical measure.
\end{lemma}

We propose the following algorithm:
\vspace{4mm}
\begin{mdframed}
\begin{align}\label{alg}
    &\text{Define}\ \hat{h}=\hat{h}_S \ \ \text{if} \ \
    \hat{R}_{\mu_{1,T}}(\hat{h}_S)-\hat{R}_{\mu_{1,T}}(\hat{h}_T) \leq c\sqrt{A_{n_T}},
    \nonumber\\
    &\text{otherwise, define}\ \hat{h}=\hat{h}_T
\end{align}
\end{mdframed}
\vspace{4mm}
\begin{theorem}\label{upper_adaptive}
    Let $\mathcal{H}$ be a hypothesis class with finite VC dimension $d_{\mathcal{H}}\geq 3.$ Furthermore, let $(\mu_0,\mu_{1,S}, \alpha)$ and $(\mu_0,\mu_{1,T}, \alpha)$ denote a source and a target problem. Suppose that there are $n_0,n_S,n_T$ i.i.d. samples from $\mu_0,\mu_{1,S},\mu_{1,T}$, respectively. Let $\delta_0, \delta>0$, $\epsilon_0=\sqrt{128\frac{d_{\mathcal{H}}\log{n_0}+\log(8/\delta_0)}{n_0}}$, $A_{n_S}$ and $A_{n_T}$ be as defined in Lemma \ref{lem_vc}. Moreover, let the outlier transfer exponent be $\rho(r)$ with coefficient $C_{\rho(r)}$ for $r\geq \epsilon_0$. Then the hypothesis $\hat{h}$ obtained by Algorithm (\ref{alg}) satisfies:
    \begin{align*}
&\mathcal{E}_{1,T}(\hat{h})\leq \min\{\mathcal{E}_{1,T}(h^*_{S,\alpha})+C\cdot A_{n_S}^{\frac{1}{2\rho(r)}},C\cdot A_{n_T}^{\frac{1}{2}}\}\nonumber\\
&R_{\mu_0}(\hat{h})\leq \alpha+\epsilon_0
\end{align*}
with probability at least $1-\delta_0-\frac{2\delta}{3}$, where $C\in(0,\infty)$ is a constant depending on $ C_{\rho(r)}, \rho(r)$.
\end{theorem}
\begin{proof}
Consider the intersection of the following events which happens with probability at least $1-\delta_0-\frac{2\delta}{3}$: 
\begin{itemize}
    \item[1)]$\{\underset{h\in \mathcal{H}}{\sup}|R_{\mu_0}(h)-\hat{R}_{\mu_0}(h)|\leq \frac{\epsilon_0}{2}\}$,
    \item[2)] The event from Lemma \ref{lem_vc} for $n_T$ number of samples generated by $\mu_{1,T}$,
    \item[3)] The event from Lemma \ref{lem_vc} for $n_S$ number of samples generated by $\mu_{1,S}$.
\end{itemize}
 Since on the considered event $h^*_{S,\alpha}\in \hat{\mathcal{H}}_0$, we get $\hat{R}_{\mu_{1,S}}(\hat{h}_S)\leq \hat{R}_{\mu_{1,S}}(h^*_{S,\alpha})$. Then using Lemma \ref{lem_vc} we obtain
\begin{align*}
    R_{\mu_{1,S}}(\hat{h}_S)-R_{\mu_{1,S}}(h^*_{S,\alpha})\leq \hat{R}_{\mu_{1,S}}(\hat{h}_S)-\hat{R}_{\mu_{1,S}}(h^*_{S,\alpha})+C\cdot A_{n_S}^{\frac{1}{2}}\leq C\cdot A_{n_S}^{\frac{1}{2}}.
\end{align*}
On the event $\{\underset{h\in \mathcal{H}}{\sup}|R_{\mu_0}(h)-\hat{R}_{\mu_0}(h)|\leq \frac{\epsilon_0}{2}\}$, we have $R_{\mu_0}(\hat{h}_S)\leq \alpha+r$ which implies that 
\begin{align*}
    R_{\mu_{1,T}}(\hat{h}_S)-R_{\mu_{1,T}}(h^*_{S,\alpha})\leq C\cdot A_{n_S}^{\frac{1}{2\rho(r)}}.
\end{align*}
Hence, 
\begin{align}\label{upp_6.8}
R_{\mu_{1,T}}(\hat{h}_S)-R_{\mu_{1,T}}(h^*_{T,\alpha})&=R_{\mu_{1,T}}(\hat{h}_S)-R_{\mu_{1,T}}(h^*_{S,\alpha})+R_{\mu_{1,T}}(h^*_{S,\alpha})-R_{\mu_{1,T}}(h^*_{T,\alpha})\nonumber\\
&\leq \mathcal{E}_{1,T}(h^*_{S,\alpha})+ C\cdot A_{n_S}^{\frac{1}{2\rho(r)}}.
\end{align}
Now if $R_{\mu_{1,T}}(\hat{h}_S)\leq R_{\mu_{1,T}}(\hat{h}_T)$, by Lemma \ref{lem_vc} the constraint in Algorithm (\ref{alg}) holds, which implies that $\hat{h}=\hat{h}_S$ and the upper bound (\ref{upp_6.8}) holds for $R_{\mu_{1,T}}(\hat{h})-R_{\mu_{1,T}}(h^*_{T,\alpha})$. On the other hand, if $R_{\mu_{1,T}}(\hat{h}_S)> R_{\mu_{1,T}}(\hat{h}_T)$, then $$R_{\mu_{1,T}}(\hat{h}_T)-R_{\mu_{1,T}}(h^*_{T,\alpha})< R_{\mu_{1,T}}(\hat{h}_S)-R_{\mu_{1,T}}(h^*_{T,\alpha}).$$ So, regardless of whether $\hat{h}=\hat{h}_S$ or $\hat{h}= \hat{h}_T$, the upper bound (\ref{upp_6.8}) holds for $R_{\mu_{1,T}}(\hat{h})-R_{\mu_{1,T}}(h^*_{T,\alpha})$.
Moreover, Since $\hat{h}$ satisfies the constraint in Algorithm (\ref{alg}), we get
\begin{align*}
    R_{\mu_{1,T}}(\hat{h})-R_{\mu_{1,T}}(h^*_{T,\alpha})&=R_{\mu_{1,T}}(\hat{h})-R_{\mu_{1,T}}(\hat{h}_T)+R_{\mu_{1,T}}(\hat{h}_T)-R_{\mu_{1,T}}(h^*_{T,\alpha})\\
    &\leq C\cdot A_{n_T}^{\frac{1}{2}}+C\cdot A_{n_T}^{\frac{1}{2}}\leq 2C\cdot A_{n_T}^{\frac{1}{2}}
\end{align*}
which completes the proof.
\end{proof}
\section{Overview of Proof of Theorem \ref{minimax_rate} (Minimax Lower Bound)}
To establish the minimax lower bound it suffices to show the following theorem.
\begin{theorem}\label{thm_sec6}
    Consider the setting of Theorem \ref{minimax_rate}. Then for any learner described in Theorem \ref{minimax_rate} that outputs a hypothesis $\hat{h}$, there exist $\mu_0,\mu_{1,S},\mu_{1,T}\in \mathcal{F}_{\mathcal{H}}(\rho,\alpha,1,\Delta)$ such that
\begin{align*}
\underset{S_{\mu_0},S_{\mu_{1,S}},S_{\mu_{1,T}}}{\mathbb{P}}\bigg(\mathcal{E}_{1,T}(\hat{h})
>c\cdot\min\{\Delta+(\frac{d_{\mathcal{H}}}{n_S})^{\frac{1}{2\rho}},(\frac{d_{\mathcal{H}}}{n_T})^{\frac{1}{2}}\}
\bigg)\geq c'
\end{align*}
where $c,c'$ are universal numerical constants.
\end{theorem}
To prove Theorem \ref{thm_sec6}, We follow Tysbakov's method \citep{Tsybakov:1315296}.
\begin{theorem}\citep{Tsybakov:1315296}\label{tsybakov}
Assume that $M\geq 2$ and the function $\text{dist}(\cdot, \cdot)$ is a semi-metric. Furthermore, suppose that $\{\Pi_{\theta_j}\}_{\theta_j\in \Theta}$ is a family of distributions indexed over a parameter space, $\Theta$,  and $\Theta$ contains elements $\theta_0, \theta_1,..., \theta_M$ such that:
\begin{enumerate}
    \item [(i)] $\text{dist}(\theta_i,\theta_j)\geq 2s>0, \ \ \forall \ 0\leq i<j\leq M$
    
    \item [(ii)] $\Pi_j\ll \Pi_0, \ \ \forall \ j=1,...,M,$ and $\frac{1}{M}\sum_{j=1}^{M} \mathcal{D}_{kl}(\Pi_j|\Pi_0)\leq \gamma \log{M}$ with $0<\gamma<1/8$ and $\Pi_j=\Pi_{\theta_j}$, $j=0,1,...,M$ and $\mathcal{D}_{kl}$ denotes the KL-divergence.
    \end{enumerate}
    Then
    \begin{align*}
\inf_{\hat{\theta}}\sup_{\theta\in \Theta}\Pi_{\theta}&(\text{dist}(\hat{\theta},\theta)\geq s)\geq \frac{\sqrt{M}}{1+\sqrt{M}}\big(1-2\gamma-\sqrt{\frac{2\gamma}{\log{M}}} \big).
    \end{align*}

\end{theorem}
We also utilize the following proposition for constructing a packing of the parameter space.
\begin{proposition}\label{gilbert}(Gilbert-Varshamov bound) Let $d \geq 8$. Then there exists a subset $\{\sigma_0,...,\sigma_M\}$ of $\{-1,+1\}^d$ such that $\sigma_0=(1,1,...,1)$,
$$\text{dist}(\sigma_j,\sigma_k)\geq \frac{d}{8}, \ \ \forall \ 0\leq j<k\leq M \ \text{and} \ M\geq 2^{d/8},$$
where $\text{dist}(\sigma, \sigma') =\text{card}({i \in [m] : \sigma(i)\neq \sigma'(i)})$ is the Hamming distance.
\end{proposition}
First note that
\begin{align*}
    \min\{\Delta+(\frac{d_{\mathcal{H}}}{n_S})^{\frac{1}{2\rho}},(\frac{d_{\mathcal{H}}}{n_T})^{\frac{1}{2}}\}&\leq 2\cdot \min\{\max\{\Delta,(\frac{d_{\mathcal{H}}}{n_S})^{\frac{1}{2\rho}}\},(\frac{d_{\mathcal{H}}}{n_T})^{\frac{1}{2}}\}\\
    &=2\cdot \max\{\min\{(\frac{d_{\mathcal{H}}}{n_S})^{\frac{1}{2\rho}},(\frac{d_{\mathcal{H}}}{n_T})^{\frac{1}{2}}\},\min\{\Delta,(\frac{d_{\mathcal{H}}}{n_T})^{\frac{1}{2}}\}\}
\end{align*}
So it suffices to show that the minimax lower bound is larger than both $\min\{(\frac{d_{\mathcal{H}}}{n_S})^{\frac{1}{2\rho}},(\frac{d_{\mathcal{H}}}{n_T})^{\frac{1}{2}}\}$ and $\min\{\Delta,(\frac{d_{\mathcal{H}}}{n_T})^{\frac{1}{2}}\}$.
We divide the proof into three parts:
\begin{itemize}
    \item Minimax lower bound is larger than $\min\{(\frac{d_{\mathcal{H}}}{n_S})^{\frac{1}{2\rho}},(\frac{d_{\mathcal{H}}}{n_T})^{\frac{1}{2}}\}$ for $d_{\mathcal{H}}\geq 17$ (see Section \ref{lower_bounsd_sec1}).
    \item Minimax lower bound is larger than $\min\{(\frac{d_{\mathcal{H}}}{n_S})^{\frac{1}{2\rho}},(\frac{d_{\mathcal{H}}}{n_T})^{\frac{1}{2}}\}$ for $16\geq d_{\mathcal{H}}\geq 3$ (see Section \ref{lower_bounsd_sec2}).
    \item Minimax lower bound is larger than $\min\{\Delta,(\frac{d_{\mathcal{H}}}{n_T})^{\frac{1}{2}}\}$ (see Section \ref{lower_bounsd_sec3}).
\end{itemize}
In each part, following Theorem \ref{tsybakov}, we construct a family of pairs of source and target distributions that belong to the class $\mathcal{F}_{\mathcal{H}}$. To accomplish this, we pick some points from the domain $\mathcal{X}$ shattered by the hypothesis class $\mathcal{H}$ and then define appropriate distributions on these points. Additionally, this family of distributions is indexed by $\{-1,+1\}^{d_{\mathcal{H}}}$, which can be treated as a metric space using the Hamming distance. To meet the requirement of condition (i) in Theorem \ref{tsybakov}, it is necessary for these indices to be well-separated, a condition that can be satisfied through utilizing Proposition \ref{gilbert}. See Appendix \ref{app_C} for the complete proof.

\newpage

\bibliography{iclr2024_conference}
\bibliographystyle{iclr2024_conference}
\newpage
\appendix
\section{Appendix A (Equivalence of Population Problems)}\label{APP_A}
We begin by stating Neyman-Pearson Lemma \citep{lehmann1986testing} for deterministic tests. In the following, a classifier $h:\cal X\rightarrow$ $\{0,1\}$ aims at classifying $H_0:\mu_0$ against the alternative $H_1:\mu_1$. In the context of hypothesis testing, $h$ is called a deterministic test. Moreover, $R_{\mu_0}(h)$ and $1-R_{\mu_1}(h)$ are called \textit{size} and \textit{power}, respectively.
\begin{theorem}[Neyman-Pearson Lemma \citep{lehmann1986testing}]\label{thm1}
Let $\mu_0$ and $\mu_1$ be probability distributions possessing densities $p_0$ and $p_1$ respectively with respect to a dominating measure $\nu$.

(i) \emph{Sufficient condition for a solution of the optimization problem (\ref{eq0})}. Let $h$ be a classifier for $H_0 : \mu_0$ against the alternative $H_1 : \mu_1$ such that for a constant $\lambda$ the followings hold
\begin{align}\label{neym1}
R_{\mu_0}(h)=\alpha    
\end{align}
and
\begin{align}\label{neym2}
    h(x)=
    \begin{cases}
        1 & \text{when }\ \ p_1(x) \geq\lambda p_0(x)\\
        0 & \text{when }\ \ p_1(x) <\lambda p_0(x)
    \end{cases}
\end{align}
Then $h$ is a solution of (\ref{eq0}).

(ii) \emph{Necessary condition for a solution of the optimization problem (\ref{eq0})}. Suppose that there exist a classifier $h$ and a constant $\lambda$ such that (\ref{neym1}) and (\ref{neym2}) hold. Then, any solution of (\ref{eq0}), denoted by $h^*$, satisfies the following a.s. $\nu$
\begin{align}\label{neym3}
    h^*(x)=
    \begin{cases}
        1 & \text{when }\ \ p_1(x) >\lambda p_0(x)\\
        0 & \text{when }\ \ p_1(x) <\lambda p_0(x)
    \end{cases}
\end{align}
$h^*$ also satisfies $R_{\mu_0}(h^*)=\alpha$ unless there exists a classifier $h'$ with $R_{\mu_0}(h')<\alpha$ and $R_{\mu_1}(h')=0$.
\end{theorem}
\begin{remark}
To obtain a solution of the optimization problem (\ref{eq1})
 in the source (or (\ref{eq2}) in the target) it suffices to take the level set $\mathcal{L}_{\lambda}^{S}$ ($\mathcal{L}_{\lambda'}^{T}$ in the target) whose measure under $\mu_0$ is $\alpha$, and then define a classifier $h$ as $h=\mathbb{1}_{\mathcal{L}_{\lambda}^{S}}$ in the source ($h=\mathbb{1}_{\mathcal{L}_{\lambda'}^{T}}$ in the target). Obviously, the classifier $h$ satisfies (\ref{neym1}) and (\ref{neym2}), and therefore it is a solution of (\ref{eq1}).
\end{remark}

\begin{proposition}\label{pro_ney}
    Suppose that there exist a classifier $h$ with $h\equiv 1$ on $\{x:p_0(x)=0\}$ a.s. $\nu$ and a constant $\lambda$ such that (\ref{neym1}) and (\ref{neym2}) hold for $h$. Then any solution of (\ref{eq0}), denoted by $h^*$, such that $h^*\equiv 1$ on $\{x:p_0(x)=0\}$ a.s. $\nu$ and $R_{\mu_0}(h^*)=\alpha$ satisfies (\ref{neym2}) a.s. $\nu$.
\end{proposition}
\begin{proof}
Let $S_1=\{x:h^*(x)=h(x)\},\ S_2=\{x:h^*(x)\neq h(x), p_1(x)\neq \lambda p_0(x)\},\ S_3= \{x:h^*(x)\neq h(x), p_1(x)=\lambda p_0(x)\}$. By Neyman-Pearson Lemma part (ii), we know that $\nu(S_2)=0$. Moreover, we have
\begin{align*}
    \alpha&=\int h^*p_0d\nu\\
    &=\int_{S_1} h^*p_0d\nu+\int_{S_2} h^*p_0d\nu+\int_{S_3} h^*p_0d\nu
\end{align*}
Since $\nu(S_2)=0$, we conclude that $\int_{S_2} h^*p_0d\nu=0$. Furthermore, on $S_3$ we have $h\equiv 1$ and $h\neq h^*$. So $h^*\equiv 0$ on $S_3$ and $\int_{S_3} h^*p_0d\nu=0$. Hence,
\begin{align*}
  \alpha=\int_{S_1} h^*p_0d\nu= \int_{S_1} hp_0d\nu
\end{align*}
On the other hand, we have 
\begin{align*}
    \alpha=\int hp_0d\nu&=\int_{S_1} hp_0d\nu+\int_{S_2} hp_0d\nu+\int_{S_3} hp_0d\nu\\
    &= \alpha+\int_{S_3} hp_0d\nu
\end{align*}
Therefore, $\int_{S_3} hp_0d\nu=\int_{S_3} p_0d\nu=0$, because $h\equiv 1$ on $S_3$. We claim that $\nu(S_3)=0$. By contradiction assume that $\nu(S_3)>0$. First we show that $p_0$ is positive on $S_3$ a.s. $\nu$. The reason is that on $S_3$ we have $h\equiv 1$ and $h\neq h^*$. In addition, for $x$ satisfying $p_0(x)=0$, we have $h(x)=h^*(x)=1$ a.s. $\nu$. Therefore, $p_0$ must be positive on $S_3$ a.s. $\nu$. However, we have $\int_{S_3} p_0d\nu=0$ which cannot be true since $\nu(S_3)>0$ and $p_0>0$ a.s. $\nu$ on $S_3$. Hence, we conclude that $\nu(S_3)=0$. Finally, we obtain that $\nu(S_2\cup S_3)=0$, where $S_2\cup S_3=\{x: h^*(x)\neq h(x)\}$.
\end{proof}
\begin{remark}\label{prop3.14_new}
Proposition \ref{pro_ney} implies that $\mathcal{L}^S(\alpha)$ in Definition \ref{alpha-level} is unique a.s. $\nu$.
\end{remark}
Now we are ready to prove Proposition \ref{lem6} in Section \ref{full_trans}, which characterizes equivalent source and target pairs. First we show the following technical lemma.
\begin{lemma}\label{techn}
Suppose that $\alpha$ is achievable. If $\alpha<1$, then (\ref{eq1}) cannot have any other solution $h$ with $R_{\mu_0}(h)<\alpha$ and $R_{\mu_{1,S}}(h)=0$.
\end{lemma}
\begin{proof}[Proof of Lemma \ref{techn}]
By contradiction, suppose that there exists a solution $h_1$ of (\ref{eq1}) with $R_{\mu_0}(h_1)<\alpha$ and $R_{\mu_{1,S}}(h_1)=0$. Furthermore, let $h=\mathbb{1}_{\mathcal{L}^S(\alpha)}$. Since both $h$ and $h_1$ are solutions of (\ref{eq1}), we have $R_{\mu_{1,S}}(h)=R_{\mu_{1,S}}(h_1)=0$. By Neyman-Pearson Lemma part (ii), we know that $h_1= h$ on $\{x: p_0(x)\neq \lambda p_{1,S}(x)\}$ a.s. $\nu$. Let us define the sets $S_1=\{x:h_1(x)\neq h(x)\}$ and $S_2=\{x: p_0(x)=\lambda p_{1,S}(x)\}$. Then we have $S_1\subset S_2$ a.s. $\nu$. Since $h\equiv 1$ on $S_2$, we must have $h_1\equiv 0$ on $S_1$ a.s. $\nu$. From $R_{\mu_{1,S}}(h)=R_{\mu_{1,S}}(h_1)=0$ we conclude that $\mu_{1,S}(S_1)=0$ and from $R_{\mu_0}(h_1)<R_{\mu_0}(h)=\alpha$ we conclude that $\mu_0(S_1)>0$. Let $S'_1=\{x\in S_1: p_0(x)>0\}$. Hence $\mu_0(S'_1)>0$ and $\nu(S'_1)>0$. Then, let us define the set $A=\{x: x\in S'_1, p_{1,S}(x)=0\}$. We have 
\begin{align*}
    \mu_{1,S}(S'_1)=\int_{S'_1}p_{1,S}d\nu=\int_{S'_1\backslash A}p_{1,S}d\nu=0.
\end{align*}
Since $p_{1_S}>0$ on $S_1\backslash A$, we conclude that $\nu(S'_1\backslash A)=0$ which implies that $\nu(A)>0$. On $A$, $p_{1,S}\equiv 0$ and $p_0>0$ and $h\equiv 1$ a.s. $\nu$. Hence, we should have $\lambda=0$ which implies that $\mu_0(\mathcal{L}^S(\alpha))=1$. However, we assumed that $\alpha<1$.
\end{proof}
\begin{proof}[Proof of Proposition \ref{lem6}]\label{proof_suff}
(Sufficiency) Suppose that $\mathcal{L}^S(\alpha)\in \{\mathcal{L}_{\lambda}^{T}\}_{\lambda\geq 0}$ a.s. $\nu$. Due to Neyman-Pearson Lemma part (i), $\mathbb{1}_{\mathcal{L}^S(\alpha)}$ is a solution of (\ref{eq1}). Let $h^S\in\mathcal{U}^*$ be any arbitrary solution of (\ref{eq1}). First we consider the case that $R_{\mu_{1,S}}(h^S)>0$ (or the power of $h^S$ in the source problem is less than $1$). By Neyman-Pearson Lemma part (ii) we have $R_{\mu_0}(h^S)=\alpha$. Then by Proposition \ref{pro_ney}, $h^S=\mathbb{1}_{\mathcal{L}^S(\alpha)}$ a.s. $\nu$.
%Furthermore, there exists a level set $\mathcal{L}^S_{\lambda_1}$ such that $h^S$ is $1$ on that level set and $0$ elsewhere. 
%If $\mu_S(\mathcal{L}^S_{\lambda_1})<1$ and  $\mu_0(\mathcal{L}^S_{\lambda_1})=\alpha$, then we conclude that  $\mathcal{L}^S_{\lambda_1}=\mathcal{L}^{S}(\alpha)$ a.s. $\nu$. Moreover, since $\mathcal{L}^S(\alpha)\in \{\mathcal{L}_{\lambda}^{T}\}_{\lambda\geq 0}$ a.s. $\nu$, we obtain $\mathcal{L}^S_{\lambda_1}\in \{\mathcal{L}_{\lambda}^{T}\}_{\lambda\geq 0}$ a.s. $\nu$. 
We claim that $h^S$ is a solution of (\ref{eq2}). Since $\mathcal{L}^S(\alpha)\in \{\mathcal{L}_{\lambda}^{T}\}_{\lambda\geq 0}$, there exists $\mathcal{L}^{T}_{\lambda'}$ such that $\mathcal{L}^S(\alpha)=\mathcal{L}^{T}_{\lambda'}$ a.s. $\nu$. Hence, $\mu_0(\mathcal{L}^{T}_{\lambda'})=\alpha$ and $\mathbb{1}_{\mathcal{L}^T_{\lambda'}}$ is a solution of (\ref{eq2}). Furthermore, $h^S=\mathbb{1}_{\mathcal{L}^S(\alpha)}=\mathbb{1}_{\mathcal{L}^T_{\lambda'}}$ a.s. $\nu$ which implies that $h^S$ is a solution of (\ref{eq2}).
%Let us denote $h^T $ to be the test function in the target problem associated to the level set $\mathcal{L}^{T}_{\lambda}$. Since $\mu_0(\mathcal{L}^{T}_{\lambda})=R_0(h^T)=\alpha$, Then by Neyman-Pearson lemma part (i), we conclude that $h^T$ is an optimal solution to \eqref{eq2}. Furthermore, since $h^S=h^T$ a.s. $\nu$ (because $\mathcal{L}^{S}(\alpha)=\mathcal{L}^{T}_{\lambda}$ a.s. $\nu$), we conclude that $h^S$ is also an optimal solution to \eqref{eq2}.

If $R_{\mu_{1,S}}(h^S)=0$ and $R_{\mu_0}(h^S)=\alpha$, it would be similar to the previous case. Furthermore, by Lemma \ref{techn}, we cannot have a solution $h^S$ with $R_{\mu_{1,S}}(h^S)=0$ and $R_{\mu_0}(h^S)<\alpha$.

%Define $h^S\in \mathcal{H}$ as follows: $h^S\equiv 1$ on $\mathcal{L}^S(\alpha)$ and $h^S\equiv 0$ elsewhere. Since $R_{\mu_0}(h)=\mu_0(\mathcal{L}^S(\alpha))=\alpha$, by Neyman-Pearson lemma part(i) we conclude that $h^S$ is an optimal solution(most powerful test) for \eqref{eq1}. Furthermore, since $\mathcal{L}^S(\alpha)\in \{\mathcal{L}_{\lambda}^{T}\}_{\lambda\geq 0}$, so there exists $\mathcal{L}_{\lambda}^{T}$ such that $\mathcal{L}_{\lambda}^{T}=\mathcal{L}^S(\alpha)$ a.s. $\nu$. Moreover, since $\mu_0,\mu_T\ll\nu$, we have  $\mu_0(\mathcal{L}_{\lambda}^{T})=\mu_0(\mathcal{L}^S(\alpha))=\alpha$ and $\mu_T(\mathcal{L}_{\lambda}^{T})=\mu_T(\mathcal{L}^S(\alpha))$. If we define $h^T$ as $h^T\equiv 1$ on $\mathcal{L}_{\lambda}^{T}$ and $h^S\equiv 0$ elsewhere, then by Neyman-Pearson lemma part(i), we conclude that $h^T$ is an optimal solution(most powerful test) for \eqref{eq2}. Furthermore, we have $R_{\mu_T}(h^S)=1-\mu_T(\mathcal{L}^S(\alpha))=1-\mu_T(\mathcal{L}_{\lambda}^{T})=R_{\mu_T}(h^T)$. Therefore, $h^S$ is also an optimal solution for \eqref{eq2}.  Hence, $(\mu_0,\mu_S, \alpha)$ transfers to $(\mu_0,\mu_T, \alpha)$ under $\cal H$ per Definition \ref{equ}.

(Necessity) Suppose that any solution of (\ref{eq1}) is also a solution of (\ref{eq2}). Since $\mathbb{1}_{\mathcal{L}^S(\alpha)}$ is a solution of (\ref{eq1}), we conclude that it is also a solution of (\ref{eq2}). Since $R_{\mu_0}(\mathbb{1}_{\mathcal{L}^S(\alpha)})=\alpha$, by Proposition \ref{pro_ney} and $\alpha$ achievability, $\mathbb{1}_{\mathcal{L}^T(\alpha)}=\mathbb{1}_{\mathcal{L}^S(\alpha)}$ a.s. $\nu$. Therefore, $\mathcal{L}^{S}(\alpha)=\mathcal{L}^{T}(\alpha)$ a.s. $\nu$ and $\mathcal{L}^{S}(\alpha)\in \{\mathcal{L}_{\lambda}^{T}\}_{\lambda\geq 0}$ a.s. $\nu$.

\end{proof}
\subsection{Example corresponding to Fig \ref{fig2_example}}\label{App_3}
\begin{example}\label{ex_3.4}
  Let $\nu$ be the Lebesgue measure, $\alpha=\frac{1}{16}$, and $\mathcal{U}$ be the class of all the measurable $0$-$1$ functions on $\mathbb{R}$. Furthermore, let $\mu_{1,S}= Unif[1,2]$, $\mu_{1,T}= Unif[\frac{4}{3},\frac{8}{3}]$, and 
  \[p_{0}(x)= \begin{cases} 
      \frac{x}{4}+\frac{1}{2} & -2\leq x\leq 0 \\
      \frac{-x}{4}+\frac{1}{2} & 0< x\leq 2
   \end{cases}
\]Then, we have $\mathcal{L}^S(\alpha)=\mathcal{L}^T(\alpha)=(-\infty,-2]\cup [\frac{3}{2},+\infty)$. Consider the hypothesis $h=\mathbb{1}_{\{x\in[\frac{3}{2},2]\}}\in \mathcal{U}$ which is a solution in the source but not in the target. However, source is equivalent to target under $\mathcal{U}^*$.
\end{example}
\section{Appendix B (Outlier Transfer Exponent)}
\subsection{Proof of Proposition \ref{prop_example}}
\begin{proof}
    Let 
    \begin{align*}
        &\mu_0= \mathcal{N}(0,1)\\
        &\mu_{1,S}= Unif[0,1]\\
        &\mu_{1,T}= A_1\cdot Unif[t_1,2t_0-1]+A_2\cdot Unif[2t_0-1,1]
    \end{align*}
    where $A_2>A_1>0$, $t_1>0, \frac{1}{2}<t_0<1$, $2t_0-1>t_1$ and $A_1(2t_0-t_1-1)+A_2(2-2t_0)=1$. Moreover, $$\mathcal{H}=\{\mathbb{1}_{\{x\in [a,1]\cup [b,t_0]\}}(x): t_0\leq a\leq 1, t_1\leq b\leq t_0\}.$$ Let $\alpha=\mu_0([t_0,1])$ and $r<\mu_0([2t_0-1,t_0])-\alpha$. Clearly by Neyman-Pearson Lemma the unique source and target solutions are  $h^*_{S,\alpha}=h^*_{T,\alpha}=\mathbb{1}_{\{t_0\leq x\leq 1\}}$. Then for any $h$ with $R_{\mu_0}(h)\leq \alpha+r$, $h$ is of the form $h=\mathbb{1}_{\{x\in [a,1]\cup [b,t_0]\}}$ for some $a\in[t_0,1]$ and $b\in[2t_0-1,t_0]$. Hence,
    \begin{align*}
      &R_{\mu_{1,S}}(h)-R_{\mu_{1,S}}(h^*_{S,\alpha})=a+b-2t_0,\\
      &R_{\mu_{1,T}}(h)-R_{\mu_{1,T}}(h^*_{S,\alpha})=A_2(a+b-2t_0)
    \end{align*}
    which implies that $\rho(r)=1$. However, if we take $h=\mathbb{1}_{\{2t_0-1-\epsilon\leq x\leq t_0\}}(x)$ for small enough $0<\epsilon<\frac{(1-t_0)(A_2-A_1)}{A_1}$, which violates the condition $R_{\mu_0}(h)\leq \alpha+r$, (\ref{dist}) does not hold for any $\rho<\infty$.
\end{proof}

\section{Appendix C: Proof of Theorem \ref{minimax_rate} (Minimax Lower Bound)}\label{app_C}

%class of distributions where the source and target share the same optimal solutions which is a subset of $\mathcal{F}_{\mathcal{H}}(\rho,\alpha,C,\Delta)$.

%case where
%$(\mu_0,\mu_{1_S},\alpha)$ transfers to $(\mu_0,\mu_{1_T},\alpha)$ under $\mathcal{H}$. Obviously, the class of transferable distributions is a subset of $\mathcal{F}_{\mathcal{H}}(\rho,\alpha,C,\Delta)$ because
%$\mathcal{F}_{\mathcal{H}}(\rho,\alpha,C,0)\subset \mathcal{F}_{\mathcal{H}}(\rho,\alpha,C,\Delta)$.
\subsection{Minimax lower bound is larger than $\min\{(\frac{d_{\mathcal{H}}}{n_S})^{\frac{1}{2\rho}},(\frac{d_{\mathcal{H}}}{n_T})^{\frac{1}{2}}\}$ for $d_{\mathcal{H}}\geq 17$}\label{lower_bounsd_sec1} 
Let $d=d_{\mathcal{H}}-1$ and $d_{\mathcal{H}}$ be odd (If $d_{\mathcal{H}}$ is even then define $d=d_{\mathcal{H}}-2$). Then pick $d_{\mathcal{H}}$ points $\mathcal{S}=\{x_0,x_{1,1},...,x_{1,\frac{d}{2}}, x_{2,1},..., x_{2,\frac{d}{2}}\}$ from $\mathcal{X}$ shattered by $\mathcal{H}$ (if $d_{\mathcal{H}}$ is even then we pick $d_{\mathcal{H}}-1$ points). Moreover, let $\tilde{\mathcal{H}}$ be the projection of $\mathcal{H}$ onto the set $\mathcal{S}$ with the constraint that all $h\in \tilde{\mathcal{H}}$ classify $x_0$ and $x_{-1}$ as $0$.

Next we construct a distribution $\mu_0$ and a family of pairs of distributions $(\mu_{1,S}^{\sigma},\mu_{1,T}^{\sigma})$ indexed by $\sigma\in \{-1,+1\}^{\frac{d}{2}}$. In the following, we fix $\epsilon=c_1\cdot \text{min}\{(\frac{d_{\mathcal{H}}}{n_S})^{\frac{1}{2\rho}},(\frac{d_{\mathcal{H}}}{n_T})^{\frac{1}{2}}\}$ for a constant $c_1 < 1$ to be determined.

\textbf{Distribution} $\mu_0$: We define $\mu_0$ on $\mathcal{S}$ as follows:
\begin{align*}
    \mu_0(x_{1,i})=\mu_0(x_{2,i})=\frac{2\alpha}{d} \ \ \text{for} \ \ i=1,...,\frac{d}{2}
\end{align*}
and $\mu_0(x_0)=1-2\alpha$. 

\textbf{Distribution} $\mu_{1,T}^{\sigma}$: We define $\mu_{1,T}^{\sigma}$ on $\mathcal{S}$ as follows:
\begin{align*}
   &\mu_{1,T}^{\sigma}(x_{1,i})=\frac{1}{d}+(\sigma_i/2)\cdot\frac{\epsilon}{d}\ \ \text{for} \ \ i=1,...,\frac{d}{2}\\
   &\mu_{1,T}^{\sigma}(x_{2,i})=\frac{1}{d}-(\sigma_i/2)\cdot\frac{\epsilon}{d}\ \ \text{for} \ \ i=1,...,\frac{d}{2}
\end{align*}
and $\mu_{1,T}^{\sigma}(x_0)=0$.

\textbf{Distribution} $\mu_{1,S}^{\sigma}$: We define $\mu_{1,S}^{\sigma}$ on $\mathcal{S}$ as follows:
\begin{align*}
   &\mu_{1,S}^{\sigma}(x_{1,i})=\frac{1}{d}+(\sigma_i/2)\cdot\frac{\epsilon^{\rho}}{d}\ \ \text{for} \ \ i=1,...,\frac{d}{2}\\
   &\mu_{1,S}^{\sigma}(x_{2,i})=\frac{1}{d}-(\sigma_i/2)\cdot\frac{\epsilon^{\rho}}{d}\ \ \text{for} \ \ i=1,...,\frac{d}{2}
\end{align*}
and $\mu_{1,S}^{\sigma}(x_0)=0$.

\textbf{Verifying the transfer distance condition.} For any $\sigma\in\{-1,+1\}^{\frac{d}{2}}$, let $h_{\sigma}\in\tilde{\mathcal{H}}$ be the minimizer of $R_{\mu^{\sigma}_{1,S}}$ and $R_{\mu^{\sigma}_{1,T}}$ with Type-I error w.r.t. $\mu_0$ at most $\alpha$. Then $h_{\sigma}$ satisfies the following:
\begin{equation*}
h_{\sigma}(x_{1,i})=1-h_{\sigma}(x_{2,i})=
    \begin{cases}
        1 & \text{if } \sigma_i =1\\
        0 & \text{otherwise}
    \end{cases}
    \ \ \text{for }\ \ i=1,...,\frac{d}{2}
\end{equation*}
For any $\hat{h}\in \tilde{\mathcal{H}}$ with $\alpha-\frac{2\alpha}{d}<\mu_0(\hat{h})<\alpha+\frac{2\alpha}{d}$, we have
\begin{align*}
 &\mu_{1,T}(h_{\sigma}=1)-\mu_{1,T}(\hat{h}=1)=\frac{k}{d}\cdot \epsilon\\
 &\mu_{1,S}(h_{\sigma}=1)-\mu_{1,S}(\hat{h}=1)=\frac{k}{d}\cdot \epsilon^{\rho}
\end{align*}
for some non-negative integer $k\leq \frac{d}{2}$. So the outlier transfer exponent is $\rho$ with $C_{\rho}=1$. The condition is also satisfied for $\hat{h}\in \tilde{\mathcal{H}}$ with $\mu_0(\hat{h})\leq \alpha-\frac{2\alpha}{d}$. In this case we have
\begin{align*}
 &\mu_{1_T}(h_{\sigma}=1)-\mu_{1_T}(\hat{h}=1)=\frac{k_1}{d}+\frac{k_2\epsilon}{2d}\\
 &\mu_{1_S}(h_{\sigma}=1)-\mu_{1_S}(\hat{h}=1)=\frac{k_1}{d}+\frac{k_2\epsilon^{\rho}}{2d}
\end{align*}
for some integers $k_1\leq \frac{d}{2}$ and $k_2\leq d$. Using inequality $(a+b)^{\rho}\leq 2^{\rho-1}(a^{\rho}+b^{\rho})$ the condition can be easily verified.

\textbf{Reduction to a packing.} Any classifier $\hat{h}:\mathcal{S}\rightarrow \{0,1\}$ can be reduced to a binary sequence in the domain $\{-1,+1\}^{d}$. We can first map $\hat{h}$ to $(\hat{h}(x_{1,1}),\hat{h}(x_{1,2}),...,\hat{h}(x_{1,\frac{d}{2}}),\hat{h}(x_{2,1}),...,\hat{h}(x_{2,\frac{d}{2}}))$ and then convert any element $0$ to $-1$.
We choose the Hamming distance as the distance required in Theorem \ref{tsybakov}. By applying Proposition \ref{gilbert} we can get a subset $\Sigma$ of $\{-1,+1\}^{\frac{d}{2}}$ with $|\Sigma|=M\geq 2^{d/16}$ such that the hamming distance of any two $\sigma,\sigma'\in\Sigma$ is at least $d/16$. Any $\sigma,\sigma'\in \Sigma$ can be mapped to binary sequences in the domain $\{+1,-1\}^d$ by replicating and negating, i.e., $(\sigma,-\sigma),(\sigma',-\sigma')\in\{+1,-1\}^d$ and the hamming distance of resulting sequences in the domain $\{+1,-1\}^d$ is at least $d/8$. Then for any $\hat{h}\in \tilde{\mathcal{H}}$ with $\mu_0(\hat{h}=1)<  \alpha+\frac{2\alpha}{d}$ and $\sigma\in \Sigma$, if the hamming distance of the corresponding binary sequence of $\hat{h}$ and $\sigma$ in the domain $\{+1,-1\}^d$ is at least $d/8$ then we have
$$\mu_{1,T}(h_{\sigma}=1)-\mu_{1,T}(\hat{h}=1)\geq \frac{d}{8}\cdot\frac{\epsilon}{d}=\frac{\epsilon}{8}$$
In particular, for any $\sigma,\sigma'\in\Sigma$ we have 
$$\mu_{1,T}(h_{\sigma}=1)-\mu_{1,T}(h_{\sigma'}=1)\geq \frac{d}{8}\cdot\frac{\epsilon}{d}=\frac{\epsilon}{8}$$

\textbf{KL divergence bound.} Define $\Pi_{\sigma}=(\mu_{1,S}^{\sigma})^{n_S}\times (\mu_{1,T}^{\sigma})^{n_T}$. For any $\sigma,\sigma'\in \Sigma$, our aim is to bound the KL divergence of $\Pi_{\sigma}, \Pi_{\sigma'}$. We have 
\begin{align*}
    \mathcal{D}_{kl}(\Pi_{\sigma}|\Pi_{\sigma'})=n_S\cdot \mathcal{D}_{kl}(\mu_{1,S}^{\sigma}|\mu_{1,S}^{\sigma'})+n_T\cdot \mathcal{D}_{kl}(\mu_{1,T}^{\sigma}|\mu_{1,T}^{\sigma'})
\end{align*}

The distribution $\mu_{1,S}^{\sigma}$ can be expressed as $P^{\sigma}_{X}\times P^{\sigma}_{Y|X}$ where $P^{\sigma}_{X}$ is a uniform distribution over the set $\{1,2,...,\frac{d}{2}\}$ and $P^{\sigma}_{Y|X=i}$ is a Bernoulli distribution with parameter $\frac{1}{2}+\frac{1}{2}\cdot(\sigma_i/2)\cdot\epsilon^{\rho}$. Hence we get
\begin{align}\label{kl-b}
     \mathcal{D}_{kl}(\mu_{1,S}^{\sigma}|\mu_{1,S}^{\sigma'})&=\sum_{i=1}^{\frac{d}{2}}\frac{1}{d/2}\cdot \mathcal{D}_{kl}\bigg(\text{Ber}(\frac{1}{2}+\frac{1}{2}\cdot(\sigma_i/2)\cdot\epsilon^{\rho})|\text{Ber}(\frac{1}{2}+\frac{1}{2}\cdot(\sigma'_i/2)\cdot\epsilon^{\rho})\bigg)\nonumber\\
     &\leq c_0\cdot \frac{1}{4}\cdot\epsilon^{2\rho}\nonumber\\
     &\leq \frac{1}{4}c_0c_1^{2\rho}\cdot\frac{d_{\mathcal{H}}}{n_S}\nonumber\\
     &\leq c_0c_1^{2\rho}\cdot \frac{d}{n_S}
\end{align}
for some numerical constant $c_0$. Using the same argument we can obtain $\mathcal{D}_{kl}(\mu_{1,T}^{\sigma}|\mu_{1,T}^{\sigma'})\leq c_0c_1^2\cdot \frac{d}{n_T}$. Hence we get 
\begin{align*}
    \mathcal{D}_{kl}(\Pi_{\sigma}|\Pi_{\sigma'})\leq 2c_0c_1d.
\end{align*}

Then, for sufficiently small $c_1$ we get $\mathcal{D}_{kl}(\Pi_{\sigma}|\Pi_{\sigma'})\leq\frac{1}{8}\log{M}$ which satisfies condition (ii) in Proposition \ref{tsybakov}.

Therefore, for any learner that outputs a hypothesis $\hat{h}$ from $\{h\in\mathcal{H}:\mu_0(h)\leq \alpha+\frac{2\alpha}{d_{\mathcal{H}}}\}$ with probability $1-\delta_0$, there exist $(\mu_0,\mu_{1,S},\mu_{1,T})\in \mathcal{F}_{\mathcal{H}}(\rho,\alpha,1,\Delta)$ such that condition on the event $\hat{h}\in \{h\in\mathcal{H}:\mu_0(h)\leq \alpha+\frac{2\alpha}{d_{\mathcal{H}}}\}$ we have

\begin{align*}
\underset{S_{\mu_0},S_{\mu_{1,S}},S_{\mu_{1,T}}}{\mathbb{P}}\bigg(\mathcal{E}_{1,T}(\hat{h})
>c\cdot\min\{\Delta+(\frac{d_{\mathcal{H}}}{n_S})^{\frac{1}{2\rho}},(\frac{d_{\mathcal{H}}}{n_T})^{\frac{1}{2}}\}
\bigg)\geq c'
\end{align*}
which implies that the unconditional probability is as follows
\begin{align*}
\underset{S_{\mu_0},S_{\mu_{1,S}},S_{\mu_{1,T}}}{\mathbb{P}}\bigg(\mathcal{E}_{1,T}(\hat{h})
>c\cdot\min\{\Delta+(\frac{d_{\mathcal{H}}}{n_S})^{\frac{1}{2\rho}},(\frac{d_{\mathcal{H}}}{n_T})^{\frac{1}{2}}\}
\bigg)\geq (1-\delta_0)c'\geq c''
\end{align*}

\subsection{Minimax lower bound is larger than $\min\{(\frac{d_{\mathcal{H}}}{n_S})^{\frac{1}{2\rho}},(\frac{d_{\mathcal{H}}}{n_T})^{\frac{1}{2}}\}$ for $16\geq d_{\mathcal{H}}\geq 3$}\label{lower_bounsd_sec2}
Pick three points $\mathcal{S}=\{x_{0},x_{1},x_{2}\}$ from $\mathcal{X}$ shattered by $\mathcal{H}$. Then we construct a distribution $\mu_0$ and two pairs of distributions $(\mu_{1,S}^k,\mu_{1,T}^k)$ for $k=-1,1$. Also fix $\epsilon=c_1\cdot \min\{(\frac{1}{n_S})^{\frac{1}{2\rho}},(\frac{1}{n_T})^{\frac{1}{2}}\}$ for a constant $c_1<1$ to be determined.

\textbf{Distribution} $\mu_0$: We define $\mu_0$ on $\mathcal{S}$ as follows:
\begin{align*}
    \mu_0(x_0)=1-2\alpha,\ \   \mu_0(x_1)=\mu_0(x_2)=\alpha
\end{align*}

\textbf{Distribution} $\mu_{1,T}^{k}$: We define $\mu_{1,T}^{k}$ on $\mathcal{S}$ as follows:
\begin{align*}
   \mu_{1,T}^{k}(x_0)=0, \ \ \mu_{1,T}^{k}(x_1)=\frac{1}{2}+\frac{k}{2}\cdot \epsilon,\ \  \mu_{1,T}^{k}(x_2)=\frac{1}{2}-\frac{k}{2}\cdot \epsilon
\end{align*}

\textbf{Distribution} $\mu_{1,S}^{k}$: We define $\mu_{1,S}^{k}$ on $\mathcal{S}$ as follows:
\begin{align*}
   \mu_{1,S}^{k}(x_0)=0, \ \ \mu_{1,S}^{k}(x_1)=\frac{1}{2}+\frac{k}{2}\cdot \epsilon^{\rho},\ \  \mu_{1,S}^{k}(x_2)=\frac{1}{2}-\frac{k}{2}\cdot \epsilon^{\rho}
\end{align*}
Let $\Pi_{k}=(\mu_{1,S}^k)^{n_S}\times (\mu_{1,T}^k)^{n_T}$ for $k=-1,1$. Then using the same argument we get $\mathcal{D}_{kl}(\Pi_{-1}|\Pi_{1})\leq c$ where $c$ is a numerical constant. Furthermore, let $h_k$ be the solution with Type-I error at most $\alpha$ for the distributions $(\mu_0,\mu_{1,S}^k)$ and $(\mu_0,\mu_{1,T}^k)$. It is easy to see that $R_{\mu_{1,T}^k}(h_{-k})-R_{\mu_{1,T}^k}(h_{k})=\epsilon$. Using Le Cam's method we get that for any $\hat{h}$ chosen from $\mathcal{H}_{\alpha}=\{h\in \mathcal{H}:\mu_0(h)\leq \alpha+ \frac{2\alpha}{3}\}$ there exist $(\mu_0,\mu_{1,S},\mu_{1,T})\in \mathcal{F}_{\mathcal{H}}(\rho,\alpha,1,0)$ such that
\begin{align*}
\underset{S_{\mu_{1,S}},S_{\mu_{1,T}}}{\mathbb{P}}\bigg(\mathcal{E}_{1,T}(\hat{h})>c\cdot \min\{(\frac{1}{n_S})^{\frac{1}{2\rho}},(\frac{1}{n_T})^{\frac{1}{2}}\}
\bigg)\geq c'
\end{align*}

Since $d_{\mathcal{H}}\leq 16$ we conclude that 

\begin{align*}
\underset{S_{\mu_{1,S}},S_{\mu_{1,T}}}{\mathbb{P}}\bigg(\mathcal{E}_{1,T}(\hat{h})>c\cdot \min\{(\frac{d_{\mathcal{H}}}{n_S})^{\frac{1}{2\rho}},(\frac{d_{\mathcal{H}}}{n_T})^{\frac{1}{2}}\}
\bigg)\geq c'
\end{align*}
for some numerical constants $c,c'$.

\subsection{Minimax lower bound is larger than $\min\{\Delta,(\frac{d_{\mathcal{H}}}{n_T})^{\frac{1}{2}}\}$}\label{lower_bounsd_sec3} We only show it for the case where $d_{\mathcal{H}}\geq 17$. The other case follows the same idea as in Section \ref{lower_bounsd_sec2}.

We follow the same idea as in the previous part. Let $\epsilon=c_1\cdot\min\{\Delta,(\frac{d_{\mathcal{H}}}{n_T})^{\frac{1}{2}}\}$ and pick the same set $\mathcal{S}$ from $\mathcal{X}$ shattered by $\mathcal{H}$ construct the distributions on $\mathcal{S}$ as follows:

\textbf{Distribution} $\mu_0$: We define $\mu_0$ on $\mathcal{S}$ as follows:
\begin{align*}
    \mu_0(x_{1,i})=\mu_0(x_{2,i})=\frac{2\alpha}{d} \ \ \text{for} \ \ i=1,...,\frac{d}{2}
\end{align*}
and $\mu_0(x_0)=1-2\alpha$.

\textbf{Distribution} $\mu_{1,T}^{\sigma}$: We define $\mu_{1,T}^{\sigma}$ on $\mathcal{S}$ as follows:
\begin{align*}
   &\mu_{1,T}^{\sigma}(x_{1,i})=\frac{1}{d}+(\sigma_i/2)\cdot\frac{\epsilon}{d}\ \ \text{for} \ \ i=1,...,\frac{d}{2}\\
   &\mu_{1,T}^{\sigma}(x_{2,i})=\frac{1}{d}-(\sigma_i/2)\cdot\frac{\epsilon}{d}\ \ \text{for} \ \ i=1,...,\frac{d}{2}
\end{align*}
and $\mu_{1,T}^{\sigma}(x_0)=0$.

\textbf{Distribution} $\mu_{1,S}^{\sigma}$: We define $\mu_{1,S}^{\sigma}$ on $\mathcal{S}$ as follows:
\begin{align*}
   &\mu_{1,S}^{\sigma}(x_{1,i})=\frac{1}{d}+(1/2)\cdot\frac{\epsilon^{\rho}}{d}\ \ \text{for} \ \ i=1,...,\frac{d}{2}\\
   &\mu_{1,S}^{\sigma}(x_{2,i})=\frac{1}{d}-(1/2)\cdot\frac{\epsilon^{\rho}}{d}\ \ \text{for} \ \ i=1,...,\frac{d}{2}
\end{align*}
and $\mu_{1,S}^{\sigma}(x_0)=0$.

Note that unlike previous part, all the distributions $\mu_{1_S}^{\sigma}$ are the same for different $\sigma$'s. 

\textbf{Verifying $\mathcal{E}_{1,T}(h^*_{S,\alpha})\leq\Delta$}. For every pair of $(\mu_{1,S}^{\sigma},\mu_{1,T}^{\sigma})$ we have
\begin{align*}
\mathcal{E}_{1,T}(h^*_{S,\alpha})\leq \frac{d}{2}\cdot \frac{\epsilon}{d}\leq \Delta
\end{align*}

verifying the transfer distance condition and reducing to a packing parts follow the same idea. We just bound the corresponding kL-divergence.

\textbf{KL divergence bound.}
Define $\Pi_{\sigma}=(\mu_{1,S}^{\sigma})^{n_S}\times (\mu_{1,T}^{\sigma})^{n_T}$. 
We have 
\begin{align*}
    \mathcal{D}_{kl}(\Pi_{\sigma}|\Pi_{\sigma'})=n_S\cdot \mathcal{D}_{kl}(\mu_{1,S}^{\sigma}|\mu_{1_S}^{\sigma'})+n_T\cdot \mathcal{D}_{kl}(\mu_{1,T}^{\sigma}|\mu_{1,T}^{\sigma'})
\end{align*}
Since source distributions are the same, the first term is zero. Following the same argument we get 
\begin{align*}
     \mathcal{D}_{kl}(\mu_{1,T}^{\sigma}|\mu_{1,T}^{\sigma'})&\leq c_0\epsilon^2\leq c_0c_1\frac{d}{n_T}
\end{align*}
where $c_0$ is the same numerical constant used in (\ref{kl-b}). Then for sufficiently small $c_1$ we get $\mathcal{D}_{kl}(\Pi_{\sigma}|\Pi_{\sigma'})\leq\frac{1}{8}\log{M}$ which satisfies condition (ii) in Proposition \ref{tsybakov}.

\end{document}